\documentclass[a4paper,12pt,reqno]{amsproc}
\usepackage[a4paper, margin=.9in]{geometry}
\usepackage[utf8]{inputenc}
\usepackage{hyperref}
\usepackage{url}
\usepackage{mathtools}
\usepackage{amsmath,amssymb,amsthm}
\usepackage{graphicx}
\usepackage{mathptmx}
\usepackage{todonotes}
\usepackage{bm}
\usepackage{mathpazo}
\usepackage[mathcal]{eucal}
\usepackage{thmtools}

\hypersetup{
colorlinks=true,
linkcolor = red,
citecolor = blue,
filecolor = blue,
urlcolor = blue
}

\declaretheorem[name=Theorem]{Thm}
\declaretheorem[within=section,name=Lemma]{Lem}
\declaretheorem[sibling=Lem,name=Assumption]{Ass}

\theoremstyle{remark}
\declaretheorem[sibling=Lem,name=Remark]{Rem}
\declaretheorem[sibling=Lem,name=Example]{Ex}

\theoremstyle{definition}
\declaretheorem[sibling=Lem,name=Definition]{Def}

\DeclarePairedDelimiterX{\KLx}[2]{(}{)}{%
  #1\;\delimsize\|\;#2%
}

\DeclarePairedDelimiter{\norm}{\lVert}{\rVert}

\newcommand{\FSeries}[1]{\mathbb{T}}

\newcommand{\Hol}[1]{#1\operatorname{-\textnormal{H{\"o}l}}}
\newcommand{\var}[1]{#1\operatorname{-var}}

\newcommand{\dataspace}{\mathbf{Met}} 
\newcommand{\barspace}{\mathbf{Bar}}

\newcommand{\PH}{\textnormal{PH}}
\newcommand{\tensorspace}{\mathbf{T}}
\newcommand{\Hom}{H}
\newcommand{\field}{\mathbb{F}}

\newcommand{\SC}{\mathrm{K}}
\newcommand{\V}{\mathbf{U}} 

\newcommand{\Barc}{\operatorname{Bar}}
\newcommand{\rone}{t}
\newcommand{\rtwo}{s}

\newcommand{\mrd}{\mathrm{d}}

\newcommand{\quot}{\hspace{-.1em}/\hspace{-.1em}}

\newcommand{\calD}{\mathcal{D}}

\newcommand{\calF}{\mathcal{F}}

\newcommand{\calH}{\mathcal{H}}
\newcommand{\calI}{\mathcal{I}}

\newcommand{\calX}{\mathcal{X}}

\newcommand{\frakT}{\mathfrak{T}}

\newcommand{\RR}{\ensuremath{\mathbb{R}}}

\newcommand{\NN}{\ensuremath{\mathbb{N}}}

\newcommand{\EE}{\ensuremath{\mathbb{E}}}

\makeatletter
\providecommand*{\diff}%
        {\@ifnextchar^{\DIfF}{\DIfF^{}}}
\def\DIfF^#1{%
        \mathop{\mathrm{\mathstrut d}}%
                \nolimits^{#1}\gobblespace
}
\def\gobblespace{%
        \futurelet\diffarg\opspace}
\def\opspace{%
        \let\DiffSpace\!%
        \ifx\diffarg(%
                \let\DiffSpace\relax
        \else
                \ifx\diffarg\[%
                        \let\DiffSpace\relax
                \else
                        \ifx\diffarg\{%
                                \let\DiffSpace\relax
                        \fi\fi\fi\DiffSpace}
\makeatother

\newcommand{\paths}{\mathrm{\mbox{\bf BV}}}

\newcommand{\Sig}{\operatorname{S}}

\newcommand{\meas}{\mathcal{M}}

\newcommand{\lsc}{\textnormal{L}}
\newcommand{\lscQ}{{\textnormal{iL}}}
\newcommand{\eul}{\chi}
\newcommand{\env}{\textnormal{E}}
\newcommand{\betti}{\beta}
\newcommand{\naive}{\textnormal{N}}

\newcommand{\SW}{\textnormal{SW}}

\newcommand{\PI}{\textnormal{PI}}

\usepackage{todonotes}
\usepackage{float}
\usepackage{graphicx}
\usepackage{subfig}
\usepackage{mathtools}
\usepackage{enumerate}
\usepackage{enumitem}

\title[Persistence paths and signature features in topological data analysis]{Persistence paths and signature features\\ in topological data analysis}

\author{Ilya Chevyrev}
\address{Mathematical Institute, University of Oxford}
\email{chevyrev@maths.ox.ac.uk}

\author{Vidit Nanda}
\address{Mathematical Institute, University of Oxford, and School of Mathematics, Institute for Advanced Study}
\email{nanda@maths.ox.ac.uk}

\author{Harald Oberhauser}
\address{Mathematical Institute, University of Oxford}
\email{oberhauser@maths.ox.ac.uk}

\begin{document}
\maketitle
\begin{abstract}
We introduce a new feature map for barcodes that arise in persistent homology computation.
The main idea is to first realize each barcode as a path in a convenient vector space, and to then compute its path signature which takes values in the tensor algebra of that vector space.
The composition of these two operations --- barcode to path, path to tensor series --- results in a feature map that has several desirable properties for statistical learning, such as universality and characteristicness, and achieves state-of-the-art results on common classification benchmarks.
\end{abstract}

\section{Introduction}

Algebraic topology provides a promising framework for extracting nonlinear features from finite metric spaces via the theory of {\em persistent homology}~\cite{ghrist:08, ns, oudot}. Persistent homology has solved a host of data-driven problems in disparate fields of science and engineering --- examples include signal processing~\cite{Perea2015}, proteomics~\cite{Gameiro15}, cosmology~\cite{CosmicWeb}, sensor networks~\cite{dsg}, molecular chemistry~\cite{fullerene} and computer vision~\cite{gait}. The typical output of persistent homology computation is called a {\em barcode}, and it constitutes a finite topological invariant of the coarse geometry which governs the shape of a given point cloud. 

\begin{figure}[h!]
\includegraphics[scale = 0.3]{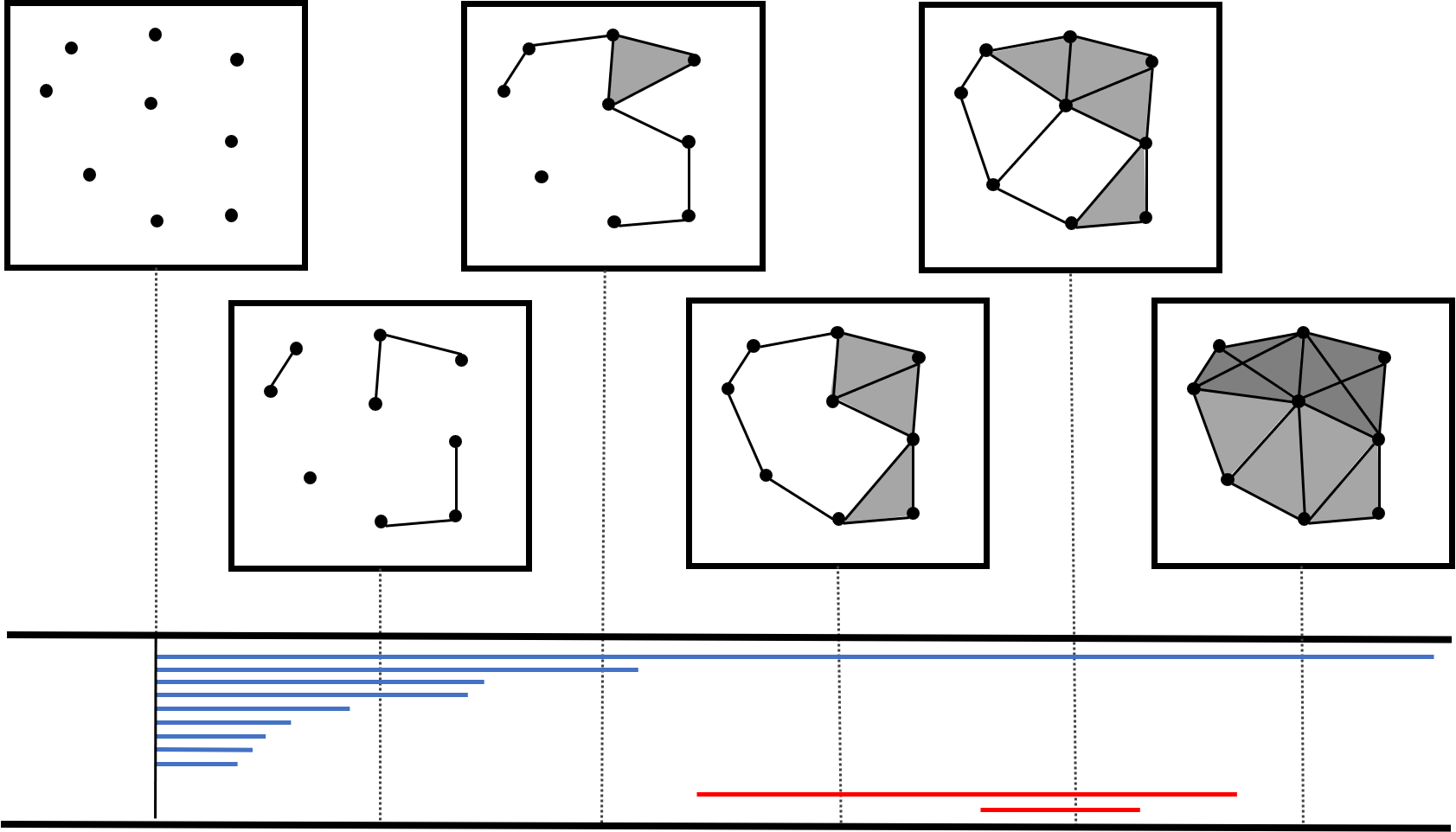}
\end{figure}

For the purposes of this introduction, it suffices to think of a barcode as a (multi)set of intervals $[b_\bullet,d_\bullet)$, each identifying those values of a scale parameter $\epsilon \geq 0$ at which some topological feature --- such as a connected component, a tunnel, or a cavity ---  is present when the input metric space is thickened by $\epsilon$.
A central advantage of persistent homology is its remarkable {\em stability} theorem~\cite[Ch.~5.6]{cdsgo}.
This result asserts that the map 
$\dataspace \to \barspace$ which assigns barcodes to finite metric spaces is 1-Lipschitz when its source and target are equipped with certain natural metrics.

\subsection*{Persistence paths and signature features}
Notwithstanding their usefulness for certain tasks, barcodes are notoriously unsuitable for standard statistical inference because $\barspace$ itself is a nonlinear metric space, and most scalable learning algorithms rely on linear methods.
In this work, we construct a feature map of the form
\[
  \Phi_\bullet:\barspace\rightarrow \tensorspace,
\]
where $\tensorspace = \tensorspace(V)$ denotes the {\em tensor algebra} of a linear space $V$.
The feature map $\Phi_\bullet$ is defined as composite, $\Phi_\bullet = \Sig \circ \iota_\bullet$, of a \emph{persistence path embedding} $\iota_\bullet$ and the \emph{path signature} $\Sig$,
\[
  \barspace \stackrel{\iota_\bullet}{\longrightarrow} \paths(V)
  \stackrel{\Sig}{\longrightarrow} \tensorspace(V),
\]
where the intermediate space $\paths(V)$ contains all continuous maps $[0,1] \to V$ of bounded variation.

\begin{description}
  \item[Persistence path embedding $\iota_\bullet$]   The maps $\iota_\bullet:\barspace \to \paths(V)$ disambiguate our $\Phi_\bullet$'s.
  There are many such embeddings, and they differ significantly in terms of stability, computability, and discriminative power.
  \item[Signature features $\Sig$]
  The map $\Sig$ represents a path as its $\tensorspace$-valued {\em signature}. 
  This map is injective (modulo natural equivalence classes of paths), provides a hierarchical description of a path, and has a rich algebraic structure that captures natural operations on paths, such as concatenation and time reversal.
\end{description}

The concept of a persistence path embedding $\iota_\bullet$ reflects the interpretation of persistent homology as a dynamic description of the topological features which appear and disappear as a metric space is thickened across various scales.
There is precedent for such constructions; e.g.~Bubenik's landscapes~\cite{Bubenik15} can be reformulated to give an important example of such an $\iota_\bullet$, which we denote with $\iota_\lscQ$.
Despite their intuitive appeal, these approaches rely ultimately on a choice of feature map for paths, on which the resulting statistical learning guarantees depend.\footnote{For example,~\cite{Bubenik15} chooses a functional on the Banach space of paths, but how to choose such functionals in a non-parametric fashion and evaluate them efficiently remains unclear (unless something special is known a priori about the probability distribution of the observed barcodes).}
We show here that the composition with the signature map resolves such issues.
For example, one of our results is that the feature map $\Phi_\lscQ=\Sig \circ
\iota_\lscQ$ is
\begin{description}
\item[universal:] non-linear functions of the data are approximated by linear functionals in feature space: for every (sufficiently regular) function $f:\barspace\to \RR$ there exists $\ell$ in the dual of $\tensorspace$, such that $f(B)\approx\langle \ell, \Phi_\lscQ(B)\rangle$ uniformly over barcodes $B$.

\item[characteristic:] the expected value of the feature map characterizes the law of the random variable:
  the map which sends a probability measure $\mu$ on $\barspace$ to its expectation $\mu \mapsto \EE_{B \sim
    \mu}[\Phi_\lscQ(B)]$ in $\tensorspace$ is injective.

\item[stable:] the map $\Phi_\lscQ: \barspace \rightarrow \tensorspace $ has explicit continuity properties, as recorded in Theorems~\ref{thm:stability} and~\ref{thm:feature} below.
\end{description}
Perhaps the biggest advantage of our approach is that it is not limited to $\iota_\lscQ$.
Besides $\iota_\lscQ$, we will also discuss the following unstable path embeddings:
\begin{description}
  
\item[the naive embedding $\iota_\naive$] sorts all intervals decreasing in length (with intervals of equal length ordered by increasing birth times), enumerates them $\{1,\ldots,n\}$, and forms an $n$-dimensional path by running in the $i$-th coordinate with unit speed if the $i$-th bar is active (otherwise remaining constant).

\item[Euler embedding $\iota_\eul$] reduces a barcode to a single Euler characteristic curve (see~\cite[Sec.~3.2]{pht}). The resulting feature map $\Phi_\eul$ is not stable, but is extremely fast to compute.

\item[Betti embedding $\iota_\betti$] records only the Betti numbers as a function of the scale, and ignores information (contained in the barcode) which connects homology across different scale values.

\item[envelope embedding $\iota_\env$] constructed by sorting the intervals $[b_\bullet,d_\bullet)$ of a barcode in descending order by length, and then assembling the ordered sequence of $b_\bullet$'s and $d_\bullet$'s into two separate paths. This appears to be a completely new embedding.

\end{description}
Analogous statements for universality and characteristicness hold for the other $\Phi_\bullet$'s.
Each of these embeddings $\iota_\bullet$ leads to different properties in terms of stability, computability, and discriminative power, for the associated feature maps $\Phi_\bullet$. For example, $\Phi_\env$ has neither the stability of $\Phi_\lscQ$ nor the computability of $\Phi_\eul$, but it gives state-of-the-art performance on supervised classification tasks.
The emergence of a single feature map which is optimal along all three axes (stability, computability, discriminative power) appears unlikely, since these requirements tend to contravene each other.
For example, stability requires the feature map to depend mostly on the longer intervals in a barcode, while in various problems (such as~\cite{Gameiro15}), the signal of interest also resides in intervals of intermediate and short length.

\subsection*{Complexity} 
The dimension of $V$ varies significantly between the different persistence path embeddings.
If a barcode contains $\approx 10^3$ intervals, $\iota_\naive$ would map it to a path that evolves in a $\approx 10^3$-dimensional space, whereas $\iota_\env$ always yields a path in $2$ dimensions.  
Each of the above feature maps $\Phi_\bullet$ gives a {\em kernel} for barcodes $k_\bullet(B,B') :=\langle \Phi_\bullet(B), \Phi_\bullet(B') \rangle$, and following~\cite{2016arXiv160108169K}, this kernel can be very efficiently computed regardless of $\dim(V)$ as long as $V$ carries an inner product.
However, for low-dimensional embeddings, $\Phi_\bullet$ can be computed directly and performs very well (e.g.~$\Phi_\env$ for the envelope embedding $\iota_\env$).
\subsection*{Benchmarks and related work}
Statistical learning from barcodes has received a lot of attention, see the background section in~\cite{2015arXiv150706217A} for a recent survey. 
The most common theme is to construct a kernel~\cite{2017arXiv170603358C} or polynomial coordinates~\cite{adcock2016ring, di2015comparing} that serve as features for barcodes.
We believe one strength of our approach is the access to both, the kernel and its feature map (at least for the Betti, Euler, and envelope embeddings; in practice the naive embedding is only accessible via kernelization due the high-dimensionality of the persistence paths); the former gives access to well-developed tools from the kernel and Gaussian processes learning literature, while the latter allows us to use {\em any} learning method such as random forests or neural networks.
A second advantage is that different choices of persistence path embeddings facilitates emphasis on different topological properties (so in a supervised learning task, the optimal $\iota_\bullet$ can be determined by cross-validation).

\subsection*{Acknowledgments}
IC is funded by a Junior Research Fellowship of St John's College, Oxford.
VN’s work is supported by The Alan Turing Institute under the EPSRC grant number EP/N510129/1, and by the Friends of the Institute for Advanced Study.
HO is supported by the Oxford-Man Institute of Quantitative Finance.
We are grateful to Steve Oudot and Mathieu Carri{\`e}re for generously sharing their data~\cite{2017arXiv170603358C} with us.

\section{Background} \label{sec:background}

As mentioned in the introduction, we construct maps of the form
\[
\barspace \to \paths(V) \to \tensorspace(V),
\]
from the space of persistence barcodes $\barspace$ to the tensor algebra $\tensorspace(V)$ of $V$ via the space of bounded variation paths $\paths$.
In this section we define these three spaces and recall important properties; see~\cite{oudot,cdsgo} resp.~\cite{lyons-04} for more details about $\barspace$ resp.~$\paths$. 
\subsection{Persistence, barcodes and stability}
\label{subsec:persistence}
The {\em Vietoris--Rips} filtration~\cite[Sec.~3.1]{ns} associates a one-parameter nested family of finite simplicial complexes $\{\SC(\rone)\}_{\rone \geq 0}$ to each finite metric space $(X,\Delta)$ via the following rule. A subset $\{x_0,\ldots,x_k\}$ of $X$ spans a $k$-dimensional simplex in $\SC(\rone)$ if and only if all the pairwise distances satisfy $\Delta(x_i,x_j) \leq \rone$. Thus, one has the inclusion $\SC(\rone) \subseteq \SC(\rtwo)$ whenever $\rone \leq \rtwo$. Computing the homology of this family~\cite[Ch.~2]{hatcher} with coefficients in a field $\field$ produces, in each dimension, a corresponding family of $\field$-vector spaces $\V_\bullet$ as follows:
\[
\V_\bullet(\rone) = \Hom_\bullet (\SC(\rone);\field),
\]  
and inclusions of simplicial complexes induce linear maps $\V_\bullet(\rone) \to \V_\bullet(\rtwo)$ for $\rone \leq \rtwo$.
This data consisting of vector spaces and linear maps indexed by real numbers is called a {\em persistence module}. 

The following result from~\cite{zc} uses the fact that the polynomial ring $\field[x]$ in one variable $x$ acts on sufficiently tame persistence modules. Since this ring is a principal ideal domain, $\field[x]$-modules have a particularly simple representation theory.
\begin{Thm}[Structure]
\label{thm:structure}
Under mild assumptions (always satisfied by Vietoris-Rips homology of finite metric spaces), each persistence module $\V$ is completely characterized up to isomorphism by a finite collection of (not necessarily distinct) intervals $B = B(\V) = \{[b_\bullet,d_\bullet)\}$, called its {\em barcode}. 
\end{Thm}
Thus, a barcode is simply a multi-set containing some subintervals of $[0,\infty]$.
When the persistence module in question comes from the Vietoris-Rips construction as described above, its barcode provides a complete summary of all the intermediate homology of $\SC(\rone)$ across $t \in [0,\infty)$.
In particular, the $i$-th {\em Betti number} of $\SC(\rone)$, written 
\begin{align}
\label{eq:betti}
\beta_i^\rone = \dim \Hom_i (\SC(\rone);\field),
\end{align} equals the number of intervals in the $i$-dimensional barcode $B_i = \Barc(\V_i)$ that contain $\rone$.
Similarly, for $t \leq s$, the rank of the induced map on homology
\begin{align}
\label{eq:rankhom}
\beta_i^{\rone,\rtwo} = \text{rank}\big(\Hom_i (\SC(\rone);\field) \to \Hom_i (\SC(\rtwo);\field)\big)
\end{align}
equals the number of intervals in $B_i$ that contain $[\rone,\rtwo]$. 

There are several efficient algorithms which take as input finite metric spaces and produce as outputs the barcodes of their Vietoris-Rips filtrations~\cite{mors, hg}. Concurrently, the theory has also developed at a rapid pace, and the following result~\cite{cseh, cdsgo} is an exemplar of its progress. (Note that $\dataspace$ is the collection of all finite metric spaces while $\barspace$ is the collection of all barcodes containing finitely many intervals.)

\begin{Thm}[Stability]
\label{thm:stability}
The map $\PH_i:\dataspace \to \barspace$ which assigns to each finite metric space its $i$-dimensional Vietoris-Rips persistent homology barcode is $1$-Lipschitz for every $i \geq 0$. Here $\dataspace$ is endowed with the Gromov-Hausdorff distance~\cite{gromov} while $\barspace$ is endowed with the bottleneck distance (as defined in~\cite[Sec.~3.1]{cseh}).
\end{Thm}

Roughly, two barcodes lie within bottleneck distance $\epsilon$ of each other if it is possible to deform one to the other by moving the endpoints of all its intervals by at most $\epsilon$ (and vice-versa). Thus, the longer intervals are more stable to perturbation of the originating metric space (in particular, intervals of length smaller than $2\epsilon$ might be created or destroyed during such a deformation).

\subsection{Paths of bounded variation}
\label{subsec:pathspace}
Let $V$ be a normed real vector space. Given a continuous path $x:[0,T] \to V$ (for some $T \geq 0$) and a finite partition of $[0,T]$
\[
  \overline{p} = (0 = p_0 \leq p_1 \leq \cdots \leq p_{\ell-1} \leq p_\ell = T),
\]
the {\em 1-variation of $x$ along $\overline{p}$} is given by
\[
\text{var}_1(x;\overline{p}) = \sum_{i=0}^{\ell-1} \norm{x(p_{i+1})-x(p_i)}_V.
\]

\begin{Def}
 The {\em total $1$-variation} of a continuous path $x:[0,T] \to V$ is defined as
  \[
  \norm{x}_{\var{1}} = \norm{x(0)}+\sup_{\overline{p}} \{\text{var}_1(x;\overline{p})\},
  \]
  where the supremum is taken over all finite partitions of $[0,T]$.
  The normed real vector space $\paths(V)$ consists of all continuous paths $x$ that satisfy $\norm{x}_{\var{1}} < \infty$, with addition and scalar multiplication being defined pointwise. The induced metric on $\paths(V)$ is given as usual by $\norm{x-y}_{\var{1}}$.
  \end{Def}
Functions of bounded variation lie strictly between Lipschitz-continuous functions
and almost everywhere differentiable functions, and in particular every
Lipschitz-continuous function $[0,T] \to V$ lies within $\paths(V)$.
\subsection{The tensor algebra}
\label{subsec:tensoralgebra}
Given a real vector space $V$ and an integer $m \geq 0$, let $V^{\otimes m} = V \otimes V \otimes \cdots \otimes V$ denote the $m$-fold tensor product of $V$ with itself.
By convention, $V^{\otimes 0} = \RR$.
The tensor algebra $\tensorspace(V)$ of $V$ is the direct product
\[
\tensorspace(V) = \prod_{m \geq 0} V^{\otimes m}.
\]
Thus, each element of $\tensorspace(V)$ is a sequence $(v_0,v_1,\ldots)$ where $v_m \in V^{\otimes m}$.
We equip $\tensorspace(V)$ with the structure of a (graded) algebra under the tensor product operation, for which $V^{\otimes k} \otimes V^{\otimes \ell}$ takes values in $V^{\otimes (k+\ell)}$.
Finally, let us emphasize that $\tensorspace(V)$ is a linear space which makes it a suitable feature space.
\begin{Ex}
  For $V=\RR^d$, $\tensorspace(V)$ is the feature space for one of the arguably most important feature maps for data in $\RR^d$: the polynomial feature map
  \begin{align}
    \label{eq:polynom}
x \mapsto (1,x,x^{\otimes 2},\ldots).
\end{align}
As we show in Section \ref{sec:persistence}, the signature map for paths in $V$ can be seen as a generalisation of the polynomial feature map \eqref{eq:polynom}.   
\end{Ex}
\section{From barcodes to paths}\label{sec:persistence}

In this section we introduce several persistence path embeddings
\[
  \iota_\bullet:\barspace \to \paths(V).
  \]
To avoid technicalities, we make the following assumption (which is always met if $B$ arises from the persistent homology of a finite metric space):\footnote{The assumption can be sufficiently mollified, but since the underlying motivation for this work is computational and finitary, we do not lose any structure of interest by restricting to tame barcodes.}
\begin{Ass}
 Assume that every barcode $B$ encountered in this section is {\em tame} in two senses: first, it has only finitely many intervals, and second, each interval is contained within $[0,T_B]$ for some sufficiently large $T_B$.
\end{Ass}

\subsection{The (integrated) landscape embedding}\label{subsec:persistence landscapes}

We first present an embedding with desirable stability properties.
The persistence landscape of a barcode $B$ is a single function $\Lambda:\NN \times \RR \to \RR$, but it is often convenient to denote each $\Lambda(k,-):\RR \to \RR$ by $\lambda_k$.
The best introduction to landscapes is visual:
\begin{figure}[h!]
\centering
\includegraphics[scale = 0.25]{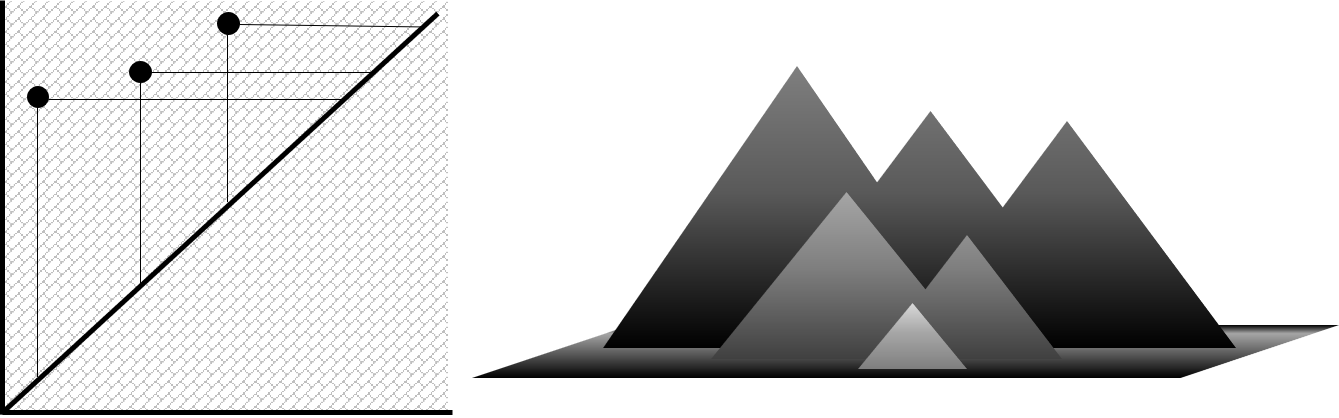}
\end{figure}

To the left is a barcode containing only three intervals, where each $[b,d)$ is shown as a point in the plane with coordinates $(b,d)$. The construction of three associated landscape functions $\lambda_1 \geq \lambda_2 \geq \lambda_3 \geq 0$, which are shown to the right, proceeds by first projecting these points onto the diagonal, and then extracting successive maximal envelopes of the resulting arrangement of line segments. The higher $\lambda_k$ for this illustrated barcode are all identically zero. 

\begin{Def}\cite[Def.~3]{Bubenik15}
\label{def:landscape}
The {\bf landscape} $\Lambda = \Lambda^B$ of the barcode $B$ is the (continuous) function $\NN \times \RR \to \RR \cup \{\infty\}$ given by
\[
\Lambda^B(k,\rone) = \lambda^B_k(\rone) = \sup\{\rtwo \geq 0 \mid \beta^{\rone-\rtwo,\rone+\rtwo} \geq k\}.
\]
Here, $\beta^{\rone-\rtwo,\rone+\rtwo}$ equals the number of intervals in $B$ which contain $[\rone-\rtwo,\rone+\rtwo]$ (see~\eqref{eq:rankhom} for the case of barcodes arising from persistent homology). Moreover, we adopt the usual convention that $\Lambda^B(k,\rone) = 0$ whenever the supremum is being taken over the empty set. 
\end{Def}

For tame barcodes $B$, one can safely exclude $\infty$ from the codomain of $\Lambda^B$. Moreover, each $\lambda^B_k$ becomes bounded and compactly supported (in addition to continuous), so we may view the assignment of landscapes to barcodes as a function 
\[
\Lambda^\bullet: \barspace \to L^p(\NN \times \RR)
\] for {\em every} $p \in [1,\infty]$.

\begin{Def}
\label{def:lscemb}
The {\bf landscape embedding} $\iota_\lsc: \barspace \to
\paths(\ell^\infty)$ assigns to each landscape $B$ a path $\iota_\lsc \left(  B \right)$ in $\paths(\ell^\infty)$ whose $k$-th component is the $k$-th landscape function for each $k \in \NN$: 
\[
  [\iota_\lsc (B)]_k(t) =  \lambda_k(t).
\] 
Similarly, the {\bf integrated landscape embedding} $\iota_\lscQ: \barspace \to \paths(\ell^\infty)$ is defined as 
\[
  [\iota_{\lscQ}(B)]_k(t) = \int_{-\infty}^t \lambda_k(s) \mrd s.
\] 
\end{Def}
The choice of $\paths(\ell^\infty)$ for the target space above is somewhat arbitrary: we may as well have mapped to $\paths(\ell^p)$ for any $p \in [1,\infty]$ or to $\paths(\RR^n)$ via truncation.  
We now show that $\iota_\lscQ$ inherits stability (in the sense of Theorem~\ref{thm:stability}) from barcodes via their landscapes. 
The two spaces defined below will appear in the proof.

\begin{Def}
For $p,q \in [1,\infty]$, define
\begin{enumerate}
 \item the Banach space $L^{p,q}(\NN\times \RR)$ consisting of all functions $y : \NN \times\RR \to \RR$ for which the following $(p,q)$-norm is finite:
$\norm{y}_{L^{p,q}} = [ \int_{\RR} ( \sum_{k=1}^\infty |y(k,t)|^p)^{q/p} \mrd t]^{1/q},$
and
\item the Sobolev path space $W^{1,q}(\RR,\ell^p)$ consisting of all functions $x:\RR \to \ell^p$ for which there exists some $\dot x \in L^q(\RR,\ell^p)$ such that $x(t) = \int_{-\infty}^t \dot x(s)\mrd s$.
The seminorm of $x$ in this case is defined by $\norm{x}_{W^{1,q}} = \norm{\dot x}_{L^q}$.
  \end{enumerate}
\end{Def}

We remark that one usually defines the Sobolev norm on $W^{1,q}$ as $\|x\|_{L^q} + \|\dot x\|_{L^q}$, while our definition drops the term $\|x\|_{L^q}$.
For paths $x$ defined on a compact interval $[0,T]$ with $x(0)=0$, we note that these norms are equivalent (but on unbounded domains, this is no longer the case).
Our choice of norm is motivated by the upcoming Lemma~\ref{lem:lscisom}.

At special values of $p$ and $q$, the two spaces defined above become more familiar. For instance, $L^{p,p}(\NN \times \RR)$ is the $L^p$ space obtained by equipping $\NN\times \RR$ with the product of the counting and Lebesgue measures, as considered in~\cite[Sec.~2.4]{Bubenik15}.
Similarly, $W^{1,\infty}(\RR, \ell^p)$ is the space of $1$-H{\"older} paths in $\ell^p$, while $W^{1,1}(\RR,\ell^p)$ is the subspace of absolutely continuous paths in $\paths(\ell^p)$. See~\cite[Sec.~1.4]{friz-victoir-book} for further details.
In any case, landscapes in the image of $\Lambda^\bullet$ lie in $L^{p,q}(\NN \times \RR)$ for every possible $p,q \in [1,\infty]$.

\begin{Lem}\label{lem:lscisom}
For all $p,q \in [1,\infty]$, the map $\calI:\Lambda(\star,\bullet) \mapsto \int_{-\infty}^{\bullet}\Lambda(\star, s)\mrd s$ is an isometry 
\[
L^{p,q}(\NN \times \RR) \stackrel{\simeq}{\longrightarrow} W^{1,q}(\RR, \ell^p).
\]
\end{Lem}
\begin{proof}
The map $\varphi : L^{p,q}(\NN\times \RR) \to L^q(\RR,\ell^p)$, $\varphi(y)(t) = \big(y(1,t), (y(2,t), \ldots\big)$, is an isometry:
\begin{equation*}
\norm{y}_{L^{p,q}}^q = \int_\RR\left(\sum_{k = 1}^\infty |y(k,t)|^p\right)^{q/p}  \mrd t ~ = ~ \int_\RR \norm{\varphi(y)(t)}_{\ell^p}^{q} \mrd t = \norm{\varphi(y)}_{L^q}^q.
\end{equation*}
By definition of $W^{1,q}(\RR,\ell^p)$, the map $\int: L^q(\RR, \ell^p) \to W^{1,q}(\RR,\ell^p)$, $x \mapsto \int_{-\infty}^\bullet x(s) \mrd s$, is also an isometry.
The conclusion follows by observing that $\calI = \int\circ\varphi$.
\end{proof}

We now obtain a desirable stability property for the integrated landscape embedding.

\begin{Thm}
\label{thm:lscembstable}
The map $\dataspace \to \paths(\ell^\infty)$ obtained by composing the $i$-dimensional Vietoris-Rips persistent homology of Theorem~\ref{thm:stability} with the integrated landscape embedding of Definition~\ref{def:lscemb}, i.e.,
\[
\dataspace \stackrel{\PH_i}{\longrightarrow} \barspace \stackrel{\Lambda^\bullet}{\longrightarrow} L^\infty(\NN \times \RR) \stackrel{\calI}{\longrightarrow} \paths(\ell^\infty),
\]
is $1$-Lipschitz for every $i \geq 0$. Here $\dataspace$ has the Gromov-Hausdorff metric while $\paths(\ell^\infty)$ has the $1$-H{\"o}lder norm
\[
\|x\|_{\Hol1} := \|x(0)\|_{\ell^\infty} + \sup_{s\neq t} \|x(t)-x(s)\|_{\ell^\infty}|t-s|^{-1}\;.
\]
\end{Thm}

\begin{proof}
It holds that $\PH_i$ is $1$-Lipschitz by Theorem~\ref{thm:stability} and $\Lambda^\bullet$ is $1$-Lipschitz by~\cite[Thm.~13]{Bubenik15}.
Taking $p=q=\infty$ in Lemma~\ref{lem:lscisom} and noting that $\Lambda^B(k,t) = 0$ for all $t\leq 0$, $k \in \NN$, and barcodes $B$, it follows that $\calI$ is also $1$-Lipschitz.
\end{proof}

\subsection{The envelope embedding}\label{subsec:envEmbed}
Consider $B \in \barspace$.
Order the intervals $\{[b_i,d_i)\}_{i=1}^m$ of $B$ in descending order by their lengths $(d_i-b_i)$ (with intervals of equal length ordered by increasing birth times), and embed them into $\RR^2$ as the (disjoint) union 
\[
U(B) = \bigcup_{i=1}^m \{i\} \times [b_i,d_i].
\] 
The {\em upper} envelope $u_B: \RR \to \RR$ of $B$ is the piecewise linear curve obtained by linearly interpolating between the highest points $(i,d_i)$ of $U(B)$ across $i \in \{0,\ldots,m\}$, with $d_0 = 0$ by convention. Similarly, the {\em lower} envelope $\ell_B$ is obtained by interpolating between the lowest points $(i,b_i)$, again with $b_0 = 0$.
Both curves are uniquely extended to have domain $\RR$ by keeping them constant on the intervals $(-\infty,0]$ and $[m,\infty)$.
We illustrate both envelopes in the accompanying figure.

\begin{figure}[h!]
\includegraphics[scale=.4]{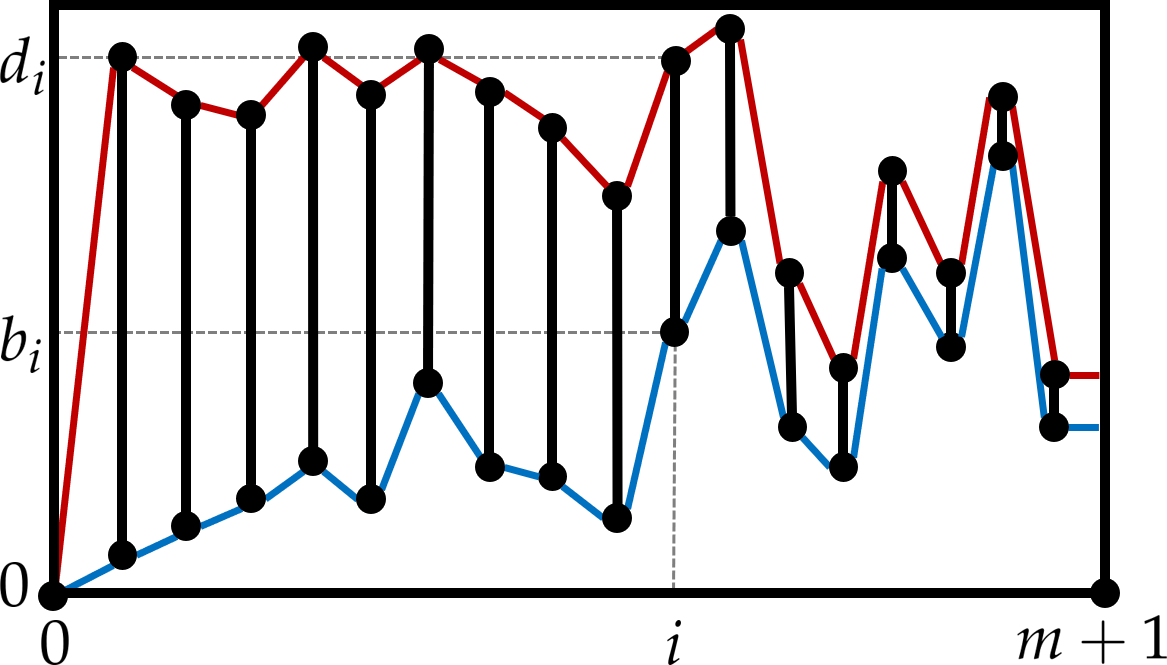}
\end{figure}
\begin{Def}
  The {\bf envelope embedding} is a map $\iota_\env:\barspace \to \paths(\RR^2)$ defined as follows.
  To each barcode $B$, it associates the path $\iota_\env(B): \RR \to \RR^2$, given by
\[
[\iota_\env(B)](t) = \left(\ell_B(t),u_B(t)\right).
\]
Here $u_B$ and $\ell_B$ are the upper and lower envelopes of $B$ as described above.
\end{Def}

For values of $t$ near zero, the upper and lower envelopes, $\ell_B(t),u_B(t)$ capture only the longest, most stable intervals of $B$.
As $t \geq 0$ increases, more of the smaller intervals get included, and the output $\iota_\env(B)$ becomes more volatile to small perturbations of $B$ (in the bottleneck metric). 
This motivates us to truncate after a given time: pick an integer $N \geq 1$ and
let $\iota_\env^N :\barspace \to \paths(\RR^2)$ be the {\em restricted envelope
  embedding} obtained by setting $\iota_\env^N (B)(t)$ equal to
\[
  \begin{cases}
																		\iota_\env(B)(t) & \text{if } t \leq N, \\
																		\iota_\env(B)(N) & \text{if } t > N, \\
													
														\end{cases}
\]
If the feature map associated to this truncated envelope embedding performs well for small values of $N$ and poorly for large ones, then one obtains evidence in favor of the hypothesis that the signal of interest genuinely resides in the larger, more stable intervals.

\subsection{The Betti and Euler embeddings}

In contrast to the previous two subsections, we now consider embeddings which depend on all homological dimensions.

\begin{Def} Denote $B = \bigcup_{i \geq 0} B_i$ where $B_i$ contains all intervals of homological dimension $i$.
Choose an integer $n \geq 1$ and numbers $a^k_i \in \RR$ for $i \geq 0$ and $1 \leq k \leq n$. Setting $a = \{a^k_i\}$, the {\bf generalised Betti embedding} $\beta(B;a) \in \paths(\RR^n)$ is defined as follows.
Let $\{t_j\}_{j=1}^m$ be the (ordered set of) all endpoints of intervals in $B$ (which lie in $[0,T]$) together with $t_1=0$ and $t_m = T$.
For each $j \in \{1,\ldots, m\}$, set
\[
\beta(B;a)(t_j) = \left(\sum_{i \geq 0} a^1_i \beta_i^{t_j}, \ldots, \sum_{i \geq 0} a^n_i \beta_i^{t_j} \right),
\]
where $\beta_i^{t_j}$ is the Betti number of $B$ as in~\eqref{eq:betti}.
We extend the definition to points $t \in (t_j,t_{j+1})$ in a piecewise linear fashion
\[
\beta(B;a)(t) = \beta(B;a)(t_j) + \frac{t - t_j}{t_{j+1}-t_j} \Big(\beta(B;a)(t_{j+1}) - \beta(B;a)(t_j)\Big).
\]

The {\bf Betti embedding} $\iota_\betti (B) \in \paths(\RR^n)$ is defined by setting $a^k_i = 1$ if $k-1=i$ and $a^k_i = 0$ otherwise:
\[
\iota_\betti (B)(t_j) = (\beta^{t_j}_0,\ldots,\beta^{t_j}_{n-1}).
\]
The {\bf Euler embedding} $\iota_\eul(B) \in \paths(\RR)$ is defined by setting $a^1_i = (-1)^i$:
\[
\iota_\eul(B)(t_j) = \sum_{i=0}^n (-1)^i \beta^{t_j}_i.
\]
\end{Def}

\begin{Rem}
When $B$ arises from the persistent homology of a finite metric space, it is computationally convenient to recall that $\iota_\eul B(t_j)$ is the Euler characteristic of the associated Vietoris-Rips simplicial complex $\SC(t_j)$, and is thus also given by an alternating count of simplices across dimension:
\[
\iota_\eul(B)(t_j) = \sum_{i \geq 0} (-1)^i \cdot \#\{i\text{-simplices in }\SC(t_j)\}.
\]
Hence, $\iota_\eul(B)$ can be computed without knowing the actual homology of $\SC$.
In fact, one could further consider generalised simplex embeddings
\[
\sigma^k(B;a)(t_j) = \sum_{i \geq 0} a^k_i \cdot \#\{i\text{-simplices in }\SC(t_j)\},
\]
which can all be computed without knowing the homology of $\SC$ (albeit do not in general capture homological invariants).
\end{Rem}

\begin{Rem}\label{rem:betti_euler}
Since $\iota_\betti$ and particularly $\iota_\eul$ are massive numerical reductions of $B$, it is apparent that metric spaces with very different persistence barcodes might have identical Betti and Euler embeddings.
On the other hand, for barcodes arising from certain popular models of random metric spaces, the expected value of $|\iota_\eul(B)(t)|$ is a remarkably good predictor of the Betti number $\beta_i^t$ for each $t \geq 0$ --- see~\cite[Sec.~5.3]{kahle:rs} and the references therein for details.
\end{Rem}

\subsection{Stability and injectivity}
As mentioned in the introduction, the emergence of a single feature map which is optimal in terms of stability, discriminative power and computability is unlikely since these three properties tend to contravene each other.
Indeed, our persistence path embeddings vary drastically in terms of stability, discriminative power, and computability.
In terms of discriminative power, the (integrated) landscape and envelope embeddings are injective maps from the space of barcodes to spaces of bounded variation paths, but neither Betti or Euler are injective, see Remark \ref{rem:betti_euler}.
In terms of stability, the only embedding that is stable is the (integrated) landscape embedding; for the other embeddings simple counterexamples can be constructed.\footnote{Consider adding small bars $[b_n,d_n)$ where $[b_n,d_n)=[0,n^{-1})$ if $n$ is even, $[b_n,d_n)=[c-n^{-1},c)$ if $n$ odd for a fixed sufficiently large $c>0$.}

\section{From paths to tensors}\label{sec:signatures}

We introduce the second component $\Sig$ for our feature map $\Phi_\bullet = \Sig \circ \iota_\bullet$.
 \begin{Def}
For a Banach space $V$, the signature map is defined as 
  \[
     \Sig: \paths(V) \rightarrow \tensorspace(V) \quad x \mapsto \left(  \Sig_0(x), \Sig_1(x), \ldots\right)
   \]
where $\Sig_0(x) = 1$ and 
 \begin{align*}
		\Sig_{m}(x) := \int_{0 < t_1 < \cdots < t_m < T} \mrd x(t_1) \otimes \mrd x(t_2) \otimes \cdots \otimes \mrd x(t_m) \in V^{\otimes m}
 \end{align*}
 is defined as a Riemann--Stieltjes integral over the $m-1$ simplex of length $T$. 
\end{Def}

 \begin{Ex}[The case $V=\RR^n$]
 For a path $x \in \paths(\RR^n)$, it holds that
 \[
 \Sig_m(x) = \sum_{i_1,\ldots,i_m} \Sig^{i_1,\ldots, i_m}(x) e_{i_1}\otimes \cdots \otimes e_{i_m} \in (\RR^n)^{\otimes m},
 \]
 where the sum is taken over all multi-indexes $(i_1,\ldots,i_m)\in \{1,\ldots,n\}^m$, $(e_{i})_{i=1}^n$ is a basis of $\RR^n$, and
 \[
 \Sig^{i_1,\ldots, i_m}(x) = \int_{0 < t_1 < \cdots < t_m < T} \mrd x_{i_1}(t_1) \mrd x_{i_2}(t_2) \cdots \mrd x_{i_m}(t_m).
 \]
Hence the term $\Sig_m(x)$ can simply be interpreted as a collection of $n^m$ real numbers.
 \end{Ex}

The mapping $\Sig$ is essentially injective up to a natural equivalence class of paths, called tree-like equivalence.\footnote{$x,y$ are tree-like equivalent iff there exists some $\RR$-tree (i.e., a metric space in which any two points are connected by a unique arc which is isometric to a real interval) $\frakT$ so that $x \ast \overleftarrow{y}$, the concatenation of $x$ with the time-reversal of $y$, decomposes as \[[0,T] \stackrel{\phi}{\longrightarrow} \frakT \stackrel{\psi}{\longrightarrow} V\]
  where $\phi$ and $\psi$ are continuous maps with $\phi(0)=\phi(T)$.
  We call a set of continuous paths in $\paths(V)$ \emph{reduced} if none of its distinct elements are tree-like equivalent.
}
\begin{Thm}{\cite{BGLY16}} \label{thm:siguniq}
  Let $x,y \in \paths(V)$. Then $\Sig(x) = \Sig(y)$ iff  $x$ and $y$ are tree-like equivalent.
\end{Thm}
Tree-like equivalence can be a useful equivalence relation, e.g.~it identifies paths that differ only by time-parametrization.\footnote{$x$ and $x'=(x(\varphi(t)))$ are tree-like equivalent for any time-change $\varphi$.}
For our applications, it is instructive to think of $\Sig_m(x)$ as the natural generalisation of the monomial of order $m$ of a vector $v \in V$ to pathspace.
\begin{Ex}
  Let $v \in \RR^n$ and consider the path $x(t)=tv$ where $t \in [0,1]$.
The components of its signature are given by
\[
  \Sig_m(x) = \int \mrd(vt_1) \otimes \cdots \mrd(v t_m) = \int v^{\otimes m} \mrd t_1 \cdots \mrd t_m = \left(\frac{v^{\otimes m}}{m!}\right),
\] 
Thus $\Sig(x)$ indeed recovers the moment map $v \mapsto (1,v,\frac{v^{\otimes 2}}{2!},\ldots)$.
\end{Ex}
A further useful similarity between the signature map and monomials is that the space of linear functions of the signature is closed under multiplication, which is commonly known as the shuffle identity.
\begin{Lem}[Shuffle identity]\label{lem:shuffle}
Suppose $\ell_i\in (V')^{\otimes m_i}$ for $i=1,2$ (where $V'$ is the continuous dual space of $V$).
Then there exists $\ell_3 \in (V')^{\otimes (m_1+m_2)}$ such that for all $x \in \paths(V)$
\[
\langle \Sig(x),\ell_1\rangle \langle \Sig(x),\ell_2\rangle = \langle \Sig(x), \ell_3 \rangle
\]
\end{Lem}
\begin{proof}
See~\cite[Thm.~2.5]{lyons-04} for an elementary proof when $V$ is finite dimensional, and~\cite[Cor.~3.9]{CDLL16} for the general case.
\end{proof}
\begin{Rem}
The linear functional $\ell_3$ is known as the shuffle product of $\ell_1$ and $\ell_2$.
\end{Rem}
In light of Lemma~\ref{lem:shuffle}, one can ask whether linear combinations of such ``monomials'' $S_m(x)$ are dense in a space of functions, and, whether the sequence of expected ``moments'' characterizes the law of the random path.
For compact subsets of $\paths(V)$, the answer to both questions is yes, as we shall see in Theorem~\ref{thm:feature}, and follows by a standard Stone--Weierstrass argument; the general, non-compact case is more subtle (cf. classical moment problem) but is known to be true under suitable integrability conditions~\cite{chevyrev2016characteristic}.

  \section{Statistical learning}\label{sec:stat learning}
 
 In this section we discuss the problem of statistical learning on the space of barcodes $\barspace$.
The space $\dataspace$ has traditionally been of more interest; however, since the persistent homology map $\PH : \dataspace \to \barspace$ from Section~\ref{subsec:persistence} is well-understood, we focus here on $\barspace$. Results for $\barspace$-valued random variables pull back along $\PH$ to results for $\dataspace$-valued random variables.

We are interested in two standard learning problems: given independent random samples $B_1, \ldots,B_k\sim B$ of a $\barspace$-valued random variable $B$, our aim is to
  \begin{enumerate}[label=(\roman*)]
  \item\label{pt:function learning} 
  learn a function $f(B)$ of
  the data $f: \barspace \to \RR$, and
  \item\label{pt:measure characteristed}
 characterize the law $\mu$ of the data $B\sim \mu$. 
  \end{enumerate}
  As mentioned in the introduction, the standard approach to both problems is to find a feature map $\Phi:\barspace \to \tensorspace(V)$ which is universal and characteristic (addressing points~\ref{pt:function learning} and~\ref{pt:measure characteristed} respectively).
Let us establish these properties for our feature map
\[
\Phi=\Sig\circ \iota :\barspace\rightarrow\tensorspace(V),
\]
where $V$ is a Banach space, $\iota$ is any persistence path embedding (e.g., one of the maps from Section~\ref{sec:persistence}) and $\Sig$ is the signature map of Section~\ref{sec:signatures}. 
Due to the injectivity of $\Sig$ (up to tree-like equivalence), $\Phi$ preserves essentially all the information captured by $\iota$.
  In particular, if $\iota$ maps some domain $\calD \subset \barspace$ injectively into the space of tree-reduced paths (as is the case for any embedding once time is added as a coordinate), then $\Phi$ is also injective on $\calD$.
  To make this precise, we use suitable quotient spaces.
  \begin{Def}
    For $\iota:\barspace\rightarrow \paths(V)$, define the equivalence relation $B \sim_\iota B'$ iff
$\iota(B)$ and $\iota(B')$ are tree-like equivalent.
Let $\barspace\quot\iota$ denote the quotient of $\barspace$ under $\sim_\iota$, and equip $\barspace\quot\iota$ with the initial topology with respect to the map
\[
\barspace\quot\iota \to \paths(V)\quot\sim_t,\quad [B] \mapsto [\iota(B)],
\]
where $\sim_t$ denotes tree-like equivalence and $\paths(V)\quot\sim_t$ is equipped with the quotient topology (recall that $\paths(V)$ bears the $1$-variation topology).
  \end{Def}
  \begin{Thm}\label{thm:feature}
   Define
    \begin{align*}
    \Phi:\barspace\quot\iota \rightarrow \tensorspace(V),\quad B\mapsto \Sig\circ \iota (B).
    \end{align*}
    On each compact subset $K \subset \barspace\quot\iota$, the map $\Phi$ has the following properties. 
\begin{enumerate}[label=(\arabic*)]
\item {(\bf Universal)} \label{pt:uni} Let $f:K \to \RR$ be continuous. For each $\epsilon>0 $, 
there exists $\ell$ in $\bigoplus_{m\geq0}
(V')^{\otimes m}$ (the dual space of the tensor algebra) such that
\[
\sup_{B\in K}|f(B)-\langle \Phi(B),
\ell \rangle | < \epsilon.
\]
\item {(\bf{Characteristic})} \label{pt:char}
Denoting by $\meas$ the set of Borel probability measures on $K$, the map
\[
\meas \to \tensorspace(V), \quad \mu \mapsto \EE_{B\sim \mu}\left[ \Phi(B) \right]
\]
is injective.
\item ({\bf Kernelized}) \label{pt:kernelised} Suppose further that $V$ is a Hilbert space.
Then the map
      \begin{align*}
      k:K\times K\rightarrow \RR, \quad k(B, B')=\langle \Phi(B),\Phi(B') \rangle
    \end{align*}
    defines a bounded, continuous kernel\footnote{A kernel on a set $\calX$ is a positive definite map $\calX\times \calX \rightarrow \RR$.
    The completion $\calH$ of $\calH_0:=\{k(x,\cdot):x \in \calX\} \subset \RR^\calX$ with respect to the inner product $\langle k(x,\cdot), k(y, \cdot)\rangle:= k(x,y)$ forms a so-called reproducing kernel Hilbert space.
      A kernel is called universal for a topological vector space $\calF\subset \RR^{\calX}$ if $\calH_0$ embeds continuously into a dense subspace of $\calF$ and called universal if the transpose map $\calF' \to \calH$ is injective, see~\cite{SGS18} for details.} which is universal for the space of continuous functions $C(K, \RR)$ and
    characteristic for Borel probability measures on $K$.
    \end{enumerate}
  \end{Thm}
  \begin{proof}
By Theorem~\ref{thm:siguniq}, the continuity of the signature map in the $1$-variation topology~\cite[Thm.~3.10]{lyons-04}, and the definition of $\barspace\quot\iota$, it follows that the map $\Phi$ is continuous and separates the points of $\barspace\quot\iota$.
Combining these properties with Lemma~\ref{lem:shuffle} shows that the set of fuctions $\{B \mapsto \langle \Phi(B), \ell \rangle: \ell \in \bigoplus_{m\ge 0} \left(  V' \right)^{\otimes m} \}$ is a point-separating subalgebra of $C(K,\RR)$, hence Point~\ref{pt:uni} follows from the Stone--Weierstrass theorem.

Point~\ref{pt:char} follows by duality: the dual of $C(K,\RR)$ are the Radon measures on $K$ (these include the Borel probability measures) and universality implies that the map $\ell \mapsto (B \mapsto \langle \Phi(B), \ell \rangle)$ is dense in this dual; hence, every Radon measure is characterized by
\begin{multline*}
\Big\{  \int_K \langle \Phi(B), \ell \rangle d\mu(B):\ell \in \bigoplus_{m=0}^\infty  \left(  V'\right)^{\otimes m} \Big\}
= \Big\{  \Big\langle \int_K \Phi(B) d\mu(B), \ell \Big\rangle :\ell \in \bigoplus_{m=0}^\infty \left( V' \right)^{\otimes m}\Big\}.
\end{multline*}

In Point~\ref{pt:kernelised} the boundedness follows from continuity of $\Phi$ and compactness of $K$.
Finally, every inner product kernel $k(B,B')=\langle \Phi(B),\Phi(B') \rangle$ is universal [resp.~characteristic] if and only if the feature map $\Phi$ is universal [resp.~characteristic]; this follows from a general argument about reproducing kernels, see for example~\cite[Prop.~E.3]{CO18Sigs} for details.
\end{proof}
The computational bottleneck for $\Phi$ is typically the calculation of the signature, since an element of $\tensorspace(V)$ truncated at level $M$ needs $O(l\dim(V)^M)$ real numbers if $\iota(B)$ is assumed piecewise linear with at most $l$ time points.
This gets prohibitively large for moderate dimensions.
By contrast, the kernel on $\barspace$, $k(B,B')=\langle\Phi(B),\Phi(B')\rangle$, is defined via the canonical inner product on $\tensorspace(V)$. 
Using~\cite{2016arXiv160108169K}, the level-$M$ approximation to $k(B,B')$ can be computed in $O(l^2c^2M)$ time\footnote{The low-rank approximation algorithm in~\cite{2016arXiv160108169K} reduces this cost further to $O(lcM)$ time and $O(M)$ memory.} and $O(l^2)$ memory,  where $c$ is the cost of evaluating one inner product in $V$. 
\subsection{Hyperparameters}\label{subsec:hyperpar}
Each of the feature maps $\Phi_\bullet$ naturally generalises to a parametrised feature map $\Phi_\bullet^\pi$ where $\pi$ denotes a set of parameters (which will be typically chosen by the learning algorithm). 
The set of parameters $\pi = \{M,\tau,\Delta,\phi\}$ are 
\begin{itemize}
\item $M \in \NN$ is the {\em truncation level} of $S$, meaning that we only consider the first $M$ components $(\Sig_0(x),\Sig_1(x),\ldots, \Sig_M(x))$ of every signature $\Sig(x)$.
\item $\tau \in \{0,1\}$ is the {\em time-augmentation} parameter. If it is non-zero, then we replace each path $x(t) \in \paths(V)$ by the path $(t,x(t)) \in \paths(\RR \times V)$ before computing its signature.
\item $\Delta \in V^l$, for some $l \geq 0$, is a {\em lag vector} containing non-negative real numbers. We replace $x(t) \in \paths(V)$ by \[\big(x(t),x(\max(t-\Delta_1,0)),\ldots,x(\max(t-\Delta_l,0))\big)\] in $\paths(V^{l+1})$ before computing its signature.
\item a {\em non-linearity} $\varphi:V\rightarrow W$. Since $\iota_\bullet$ composed with a sufficiently regular map $\varphi:V \rightarrow W$ to another vector space $W$, lifts barcodes to $\paths(W)$.
  If $\varphi$ is injective and non-linear, one expects to obtain more efficient learning from signatures of paths in $\paths(W)$ since more non-linearities are provided by $\Sig\circ \varphi \circ \iota_\bullet (B) \in \tensorspace(W)$ than by $ \Sig \circ \iota_\bullet \in \tensorspace(V)$.\footnote{Generically $W$ will be very high- or infinite-dimensional which prevents the direct calculation of $\Sig\circ \varphi \circ \iota_\bullet$.
    But if $\varphi$ is the feature map of a kernel on $V$, (i.e.~$(W,\kappa)$ is a reproducing kernel Hilbert space over $V$ with kernel $\kappa(u,v):=\langle \varphi(u),\varphi(v) \rangle$), then the signature kernelization~\cite{2016arXiv160108169K} still allows to compute the kernel $k^\pi(B,B')= \langle \Phi^\pi(B),\Phi^\pi (B') \rangle$.
  Typically we choose $V=\RR^n$ and use a classic kernel on $\RR^n$ for $\kappa$ such as the RBF kernel.}
\end{itemize}
For the choice $M= \infty$, $\tau = 0$, $\Delta = 0$ (with $l=0$) and $\varphi =\text{id}:V \to V$, the corresponding $\Phi^\pi$ recovers $\Phi$, and $k^\pi(B,B')= \langle \Phi^\pi(B),\Phi^\pi (B') \rangle$ recovers $k$.
With slight abuse of notation, we write $\Phi$ for $\Phi^\pi$ and $k$ for $k^\pi$ for the remainder of the article.

\section{Experiments}
\label{sec:experiments}
We evaluate our feature map on three supervised classification tasks: orbits, textures, and shapes. 
These are common benchmarks, taken from recent papers\footnote{
  It can be hard to replicate reported results in the literature, since the preprocessing is often not fully specified; e.g.~on OUTEX various downsampling methods combined with CBLP are possible before barcodes are computed.
  We used the same data set of barcodes and train-test split for all experiments to allow for a fair comparison.}
and are described in Figures~\ref{fig:orbit},~\ref{fig:outex}, and~\ref{fig:shapes}.
For kernel methods $k_\bullet$ we used a support vector classifier and for feature maps $\Phi_\bullet$ we used a random forest classifier.
For kernel methods we used the same Nystr\"om approximation to deal with the quadratic growth of the Gram matrix.
\subsection{Computational complexity}
For a persistence path $\iota_\bullet(B) \in \paths(\RR^n)$, truncation of the feature map $\Phi_\bullet(B)$ at tensors at level less than or equal to $M$ gives $O(n^M)$ coordinates (again, akin to classic polynomial features in machine learning). 
This combinatorial explosion limits the choice of $\iota \in \{\iota_\env, \iota_\eul, \iota_\betti\}$ for $\Phi_\iota$ in practice: the integrated landscape embedding $\iota_\lscQ$ produces paths in $\paths(\RR^n)$ for large $n$, which rules out $\Phi_\lscQ$.
However, the kernelization $k_\bullet$ needs constant memory and computation time $O(Mn)$; in practice this allows us to evaluate up to $n \approx 10^3$ on standard laptops~\cite{2016arXiv160108169K}; thus $k_\bullet$ can be efficiently computed for all $\iota_\bullet$'s as long as the persistence paths are not too long (long persistence paths are in principle possible by using the low-rank algorithm from~\cite{2016arXiv160108169K} which we did not implement).
On the other hand $\Phi_\env(B)$, $\Phi_\eul(B)$, and $\Phi_\betti(B)$ are fast to compute directly since their $\iota$'s produce low-dimensional paths.
\subsection{Implementation}
We implement our feature map $\Phi_\bullet$ and the kernel $k_\bullet$ in Python's sci-kit learn package~\cite{scikit-learn}.
As part of this, we use for the signature kernel computation $k_\bullet$ the legacy code from~\cite[Alg.~3]{2016arXiv160108169K}, but we did not implement the low-rank algorithm (which limits $k_\bullet$ to paths with $\le 300$ time ticks on laptops, which is the reason why we do not report results for $k_\eul$ and $k_\betti$ on the Orbits dataset); for $\Phi_\bullet|_M=\Sig|_M\circ \iota_\bullet$ we make use of Terry Lyons' ESIG package to compute the truncated signature $\Sig|_M$.

To allow for a fair comparison, we also implemented the sliced Wasserstein kernel $k_{\SW}$~\cite{2017arXiv170603358C} and the persistence image feature map~\cite{adams2017persistence} combined with a random forest.
This allows for a reasonable benchmarking since~\cite{2017arXiv170603358C,adams2017persistence} compares performance of $k_{\SW}$ and $\Phi_{\PI}$ against a number of other methods --- see~\cite[Sec.~4.2 and~4.3]{2017arXiv170603358C} and~\cite[Sec.~6]{adams2017persistence}.

\subsection{Hyperparameter tuning}
We use a grid search and $5$-fold cross-validation for parameter tuning.
The hyperparameter grid for $\Phi^\pi_\bullet$ consists of parameters $\pi$ that are: the truncation level $M$ in the tensor algebra and time-augmentation parameter $\tau$; additionally the parameters of the classifiers were used, and if the kernelized version is used, then additionally the parameters of the kernel on $V$ (we used throughout the Gauss kernel so this is just the length scale $\sigma >0$).
For the envelope embedding $\env$ a further grid search was performed over the restriction parameter $N$ described in Section~\ref{subsec:envEmbed}.
Adding lags described in Section~\ref{subsec:hyperpar} is possible but was not tested here.
\subsection{Performance}
Table~\ref{table_accuracy} reports the mean accuracy of repeating each experiment 20 times together with the standard deviation.
As benchmark we report performance of the sliced Wasserstein kernel $k_{\SW}$ and the persistence image features $\Phi_{\PI}$.
For $k_{\SW}$, we used the approximation given in~\cite[Alg.~1]{2017arXiv170603358C} with six directions.
For $\Phi_{\PI}$, we used the \texttt{persim} package\footnote{\url{https://github.com/sauln/persim}} with the Gauss kernel, linear weight function, and a grid search over the number of pixels and variance $\sigma>0$.
\begin{table}[h!]
\renewcommand{\arraystretch}{1.3}
\caption{Mean accuracy ($\pm$ standard deviation). }
\label{table_accuracy}
\centering
            \begin{tabular}{l|l|l|l}
             Method & Textures & Orbits & Shapes\\ 
              \hline\hline
              $k_{\SW}$&  $96.8 \pm 1.0$   &  $94.6\pm 1.3$ &   $95.8\pm 1.6$\\
              $\Phi_\PI$ &  $93.7\pm 1.0$   &  $\mathbf{99.86\pm 0.21}$ &   $90.3\pm 2.3$\\
              \hline
              $k_\env$  &  $90.4 \pm 1.5$&   $96.6\pm 0.9$&   $92.7\pm 1.5$ \\
              $k_\eul$  &  $94.9 \pm 0.6$&   NA &   $92.4\pm3.0$ \\ 
              $k_\betti$ & $\mathbf{97.8\pm 0.2}$ & NA & $93.0\pm 3.0$\\
              $\Phi_\env$ & $88.1\pm0.8$ & $98.1 \pm 1.0$ & $95.0 \pm 0.9$\\  
              $\Phi_\eul$ & $92.9 \pm 0.7 $ & $98.8 \pm 0.6 $ & $98.0 \pm 1.1 $\\
              $\Phi_\betti$ & $96.6 \pm 0.6 $ & $97.7 \pm 0.8 $ & $\mathbf{98.1 \pm 0.7}$
            \end{tabular}
\end{table}

Table~\ref{table_accuracy} shows that our approach performs very competitively, achieving state-of-the art in two common benchmarks.
To the best of our knowledge, the Betti embedding $\iota_\betti$ beats the state-of-the-art for shapes in feature form and achieves close to state-of-the-art for textures in kernelized form.
This is encouraging since both $k_{\SW}$ and $\Phi_{\PI}$ provide very competitive benchmarks.
\begin{Rem}
 Our implementation of $k_{\SW}$ achieves better accuracy, particularly for Orbits, than reported in the original papers~\cite{2017arXiv170603358C}, where it is $96.1 \pm 0.4$ for Textures and $83.7 \pm 0.5$ for Orbits.
We believe this is due to a different Gram matrix approximation (we use Nystr\"om for all kernel methods).
Similarly $\Phi_{\PI}$ drastically outperformed the results for the same orbits experiment in~\cite{adams2017persistence}; we believe this is due to our choice of classifier, namely a random forest (vs. discriminant subspace ensemble).
\end{Rem}
The best results for $k_\SW$, $\Phi_\PI$, and $\env$ were achieved by looking at the 0th homology for texture and orbits, and 1st homology for shapes (we did not combine homologies of different dimensions for these methods).

The values $M$ for the optimal truncation level in the tensor algebra (as chosen by cross-validated gridsearch) is given in Table~\ref{Tab:results all};
we ran all tests for $1 \leq M \leq 8$.
The features $\Phi_\bullet$ performed best at higher levels, while $k_\bullet$ performed best at lower levels. 
We suspect that the reason is that the Gaussian kernel non-linearity that lifts barcodes to paths in infinite-dimensional spaces, allows to capture the needed information already on level $2$ or $3$.
We also ran the same experiments for the landscape\footnote{We employed~\cite[Alg.~1]{landtool} to compute persistence landscapes}
and the naive embedding, $\iota_\lsc$ and $\iota_\naive$, but the results were not competitive on either dataset.

\begin{table}[h!]
\renewcommand{\arraystretch}{1.3}
       \caption{Best truncation level $M$.}
            \label{Tab:results all}

            \centering
            \begin{tabular}{l|l|l|l}
              $M$& Textures & Orbits & Shapes\\ 
              \hline\hline
              $k_\env$  &  $2$&   $2$&   $2$ \\
              $k_\eul$  &  $3$&   NA &   $2$ \\ 
              $k_\betti$ & $3 $ & NA & $2$\\
              $\Phi_\env$ & $5$ & $8$ & $4$\\  
              $\Phi_\eul$ & $8$ & $7$ & $7$\\
              $\Phi_\betti$ & $6$ & $8$ & $5$
            \end{tabular}
             \end{table}

 \begin{figure}[h!]%
    \includegraphics[width=5cm]{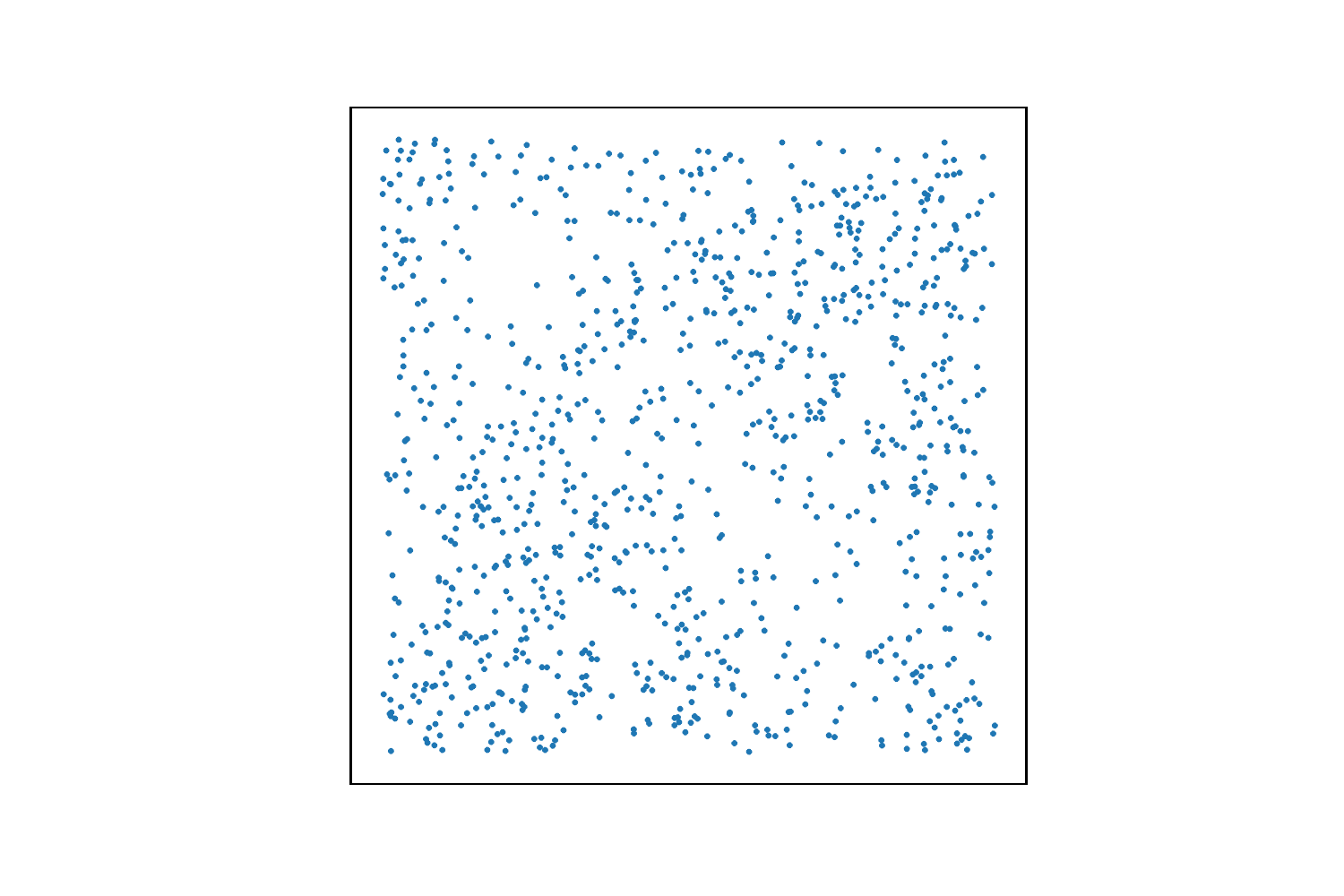} 
    \includegraphics[width=5cm]{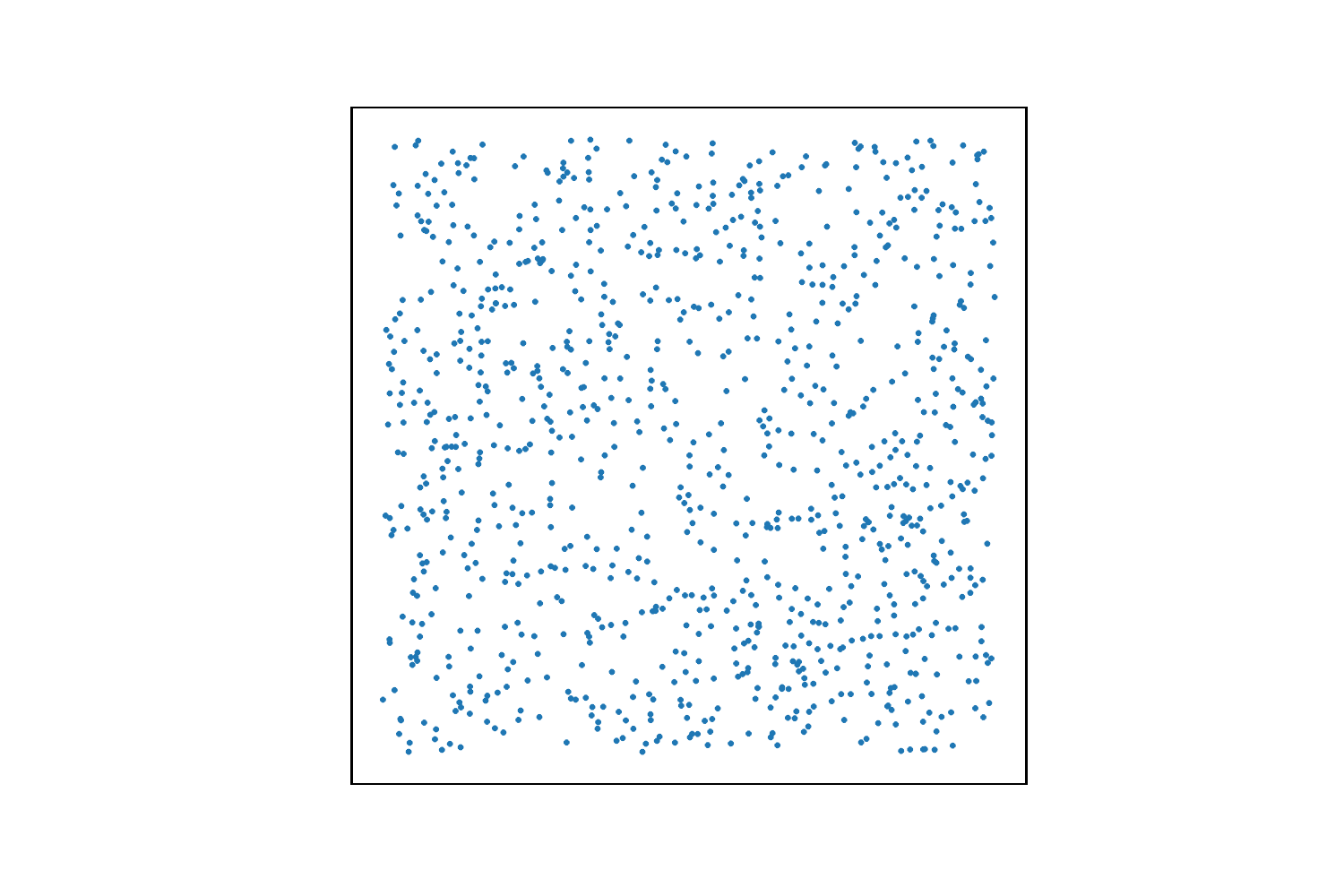}
    \includegraphics[width=5cm]{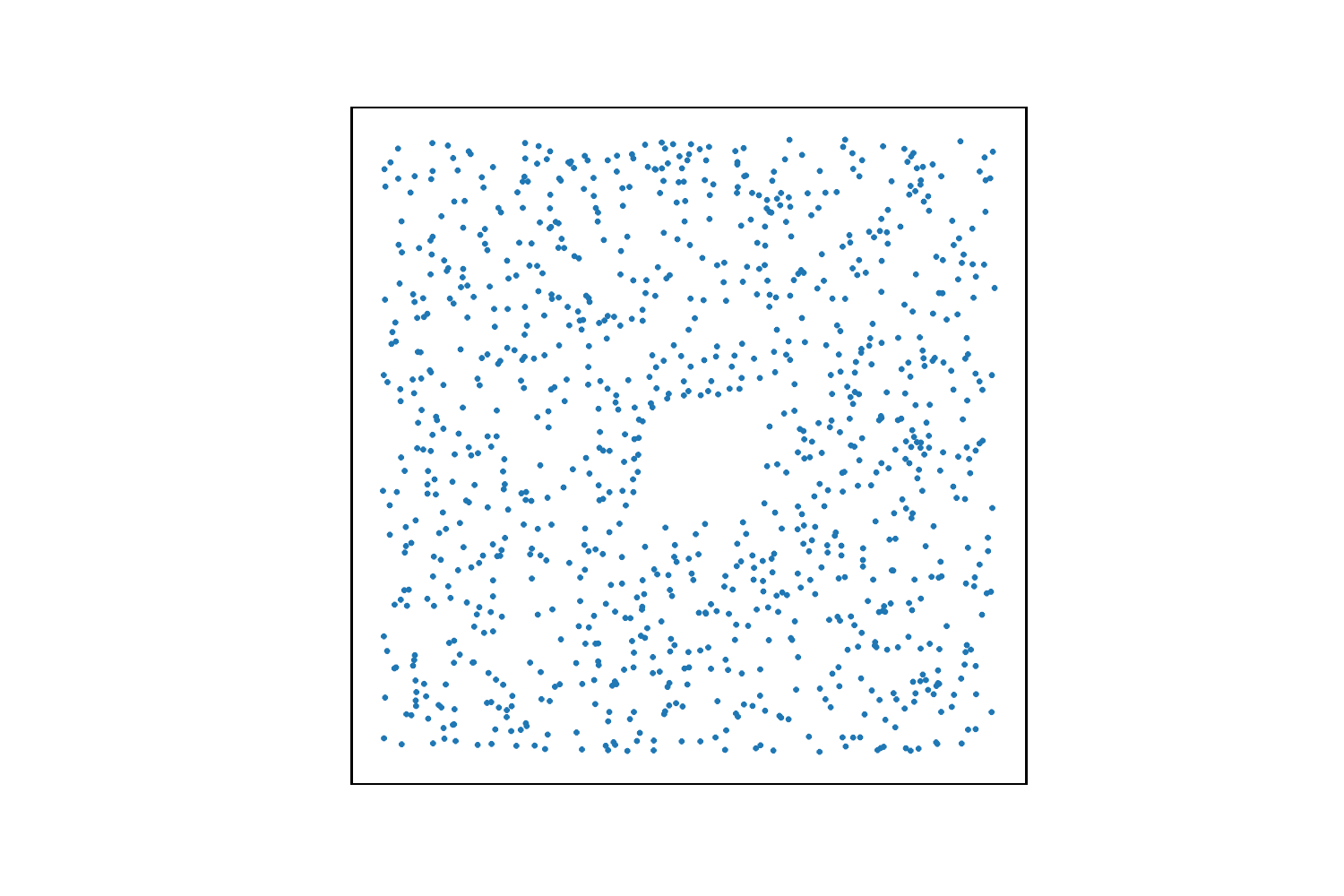}
    \includegraphics[width=5cm]{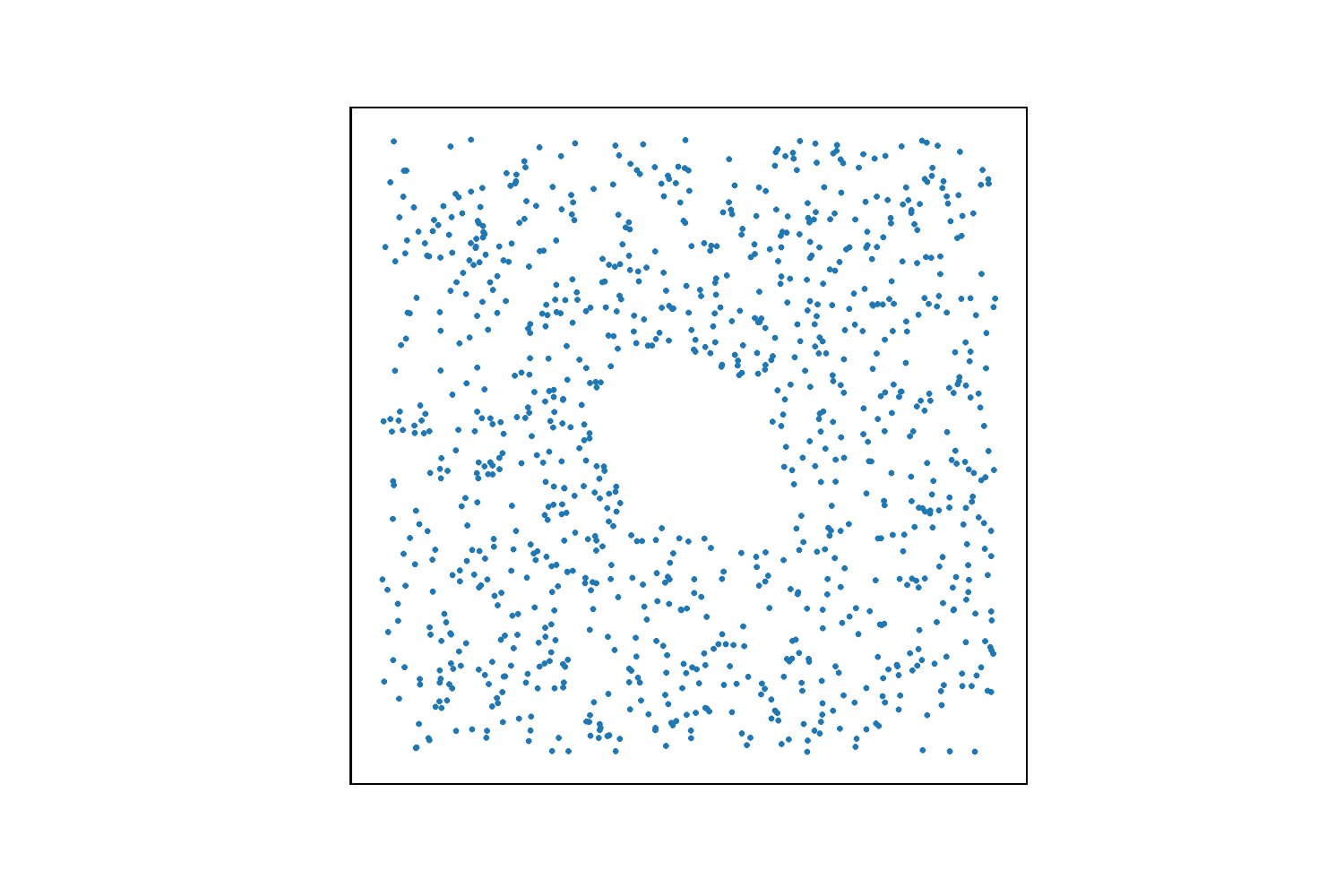}
    \includegraphics[width=5cm]{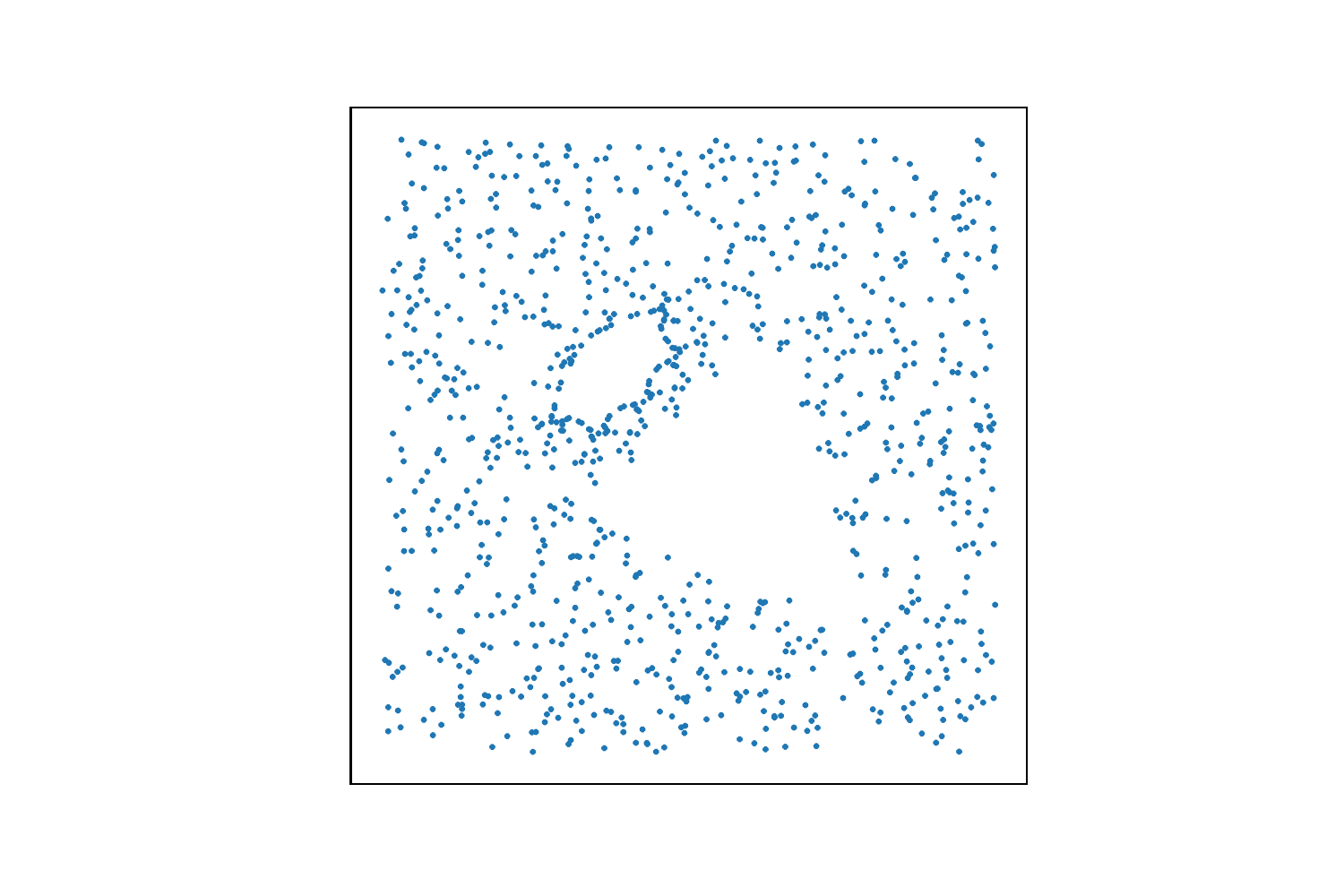}
    \caption{{\bf Orbits.} The dataset consists of $500$ orbits.
      Each orbit is a set of $1001$ points in $\RR^2$, $\{(x_n,y_n):\,n=0,\ldots,1000\}\subset \RR^2$, generated by one out of five discrete dynamical systems with a random initial value $(x_0,y_0)$.
      The five dynamical systems are given by taking the parameter $r\in \{2.5, 3.5, 4, 4.1, 4.3\}$, which thus acts as label, and update rule $x_{n+1} = x_n + r y_n (1-y_n) \mod 1 $, and $y_{n+1} = y_n + r x_n (1-x_{n}) \mod 1 $ where $(x_0,y_0)$ is chosen uniformly at random in $(0,1)^2$. 
      For each of the labels $r$ we generated $100$ orbits and used $50\%$ to $50\%$ as train-test split for the resulting $500$ orbits.
      Above shows one orbit for each value of $r=2.5,3.5,4,4.1,4.3$ (left to right and top to bottom).
      }
    \label{fig:orbit}%
\end{figure}

\begin{figure}[h!]%

    \centering
    \subfloat{{\includegraphics[width=4cm]{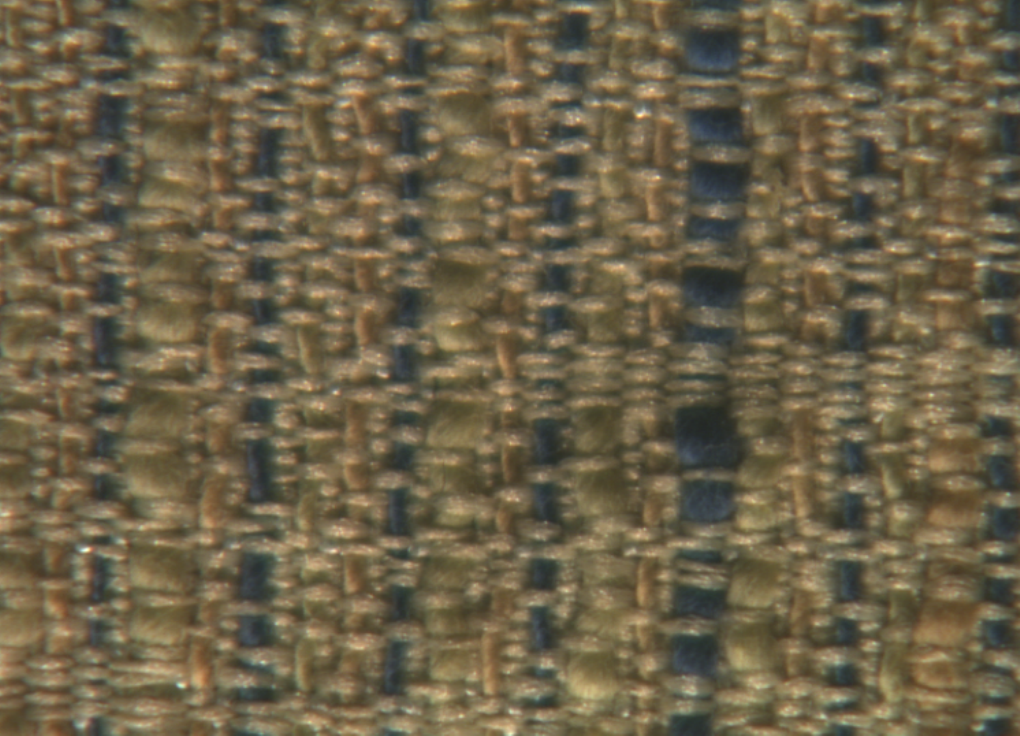} }}%
    \qquad
    \subfloat{{\includegraphics[width=4cm]{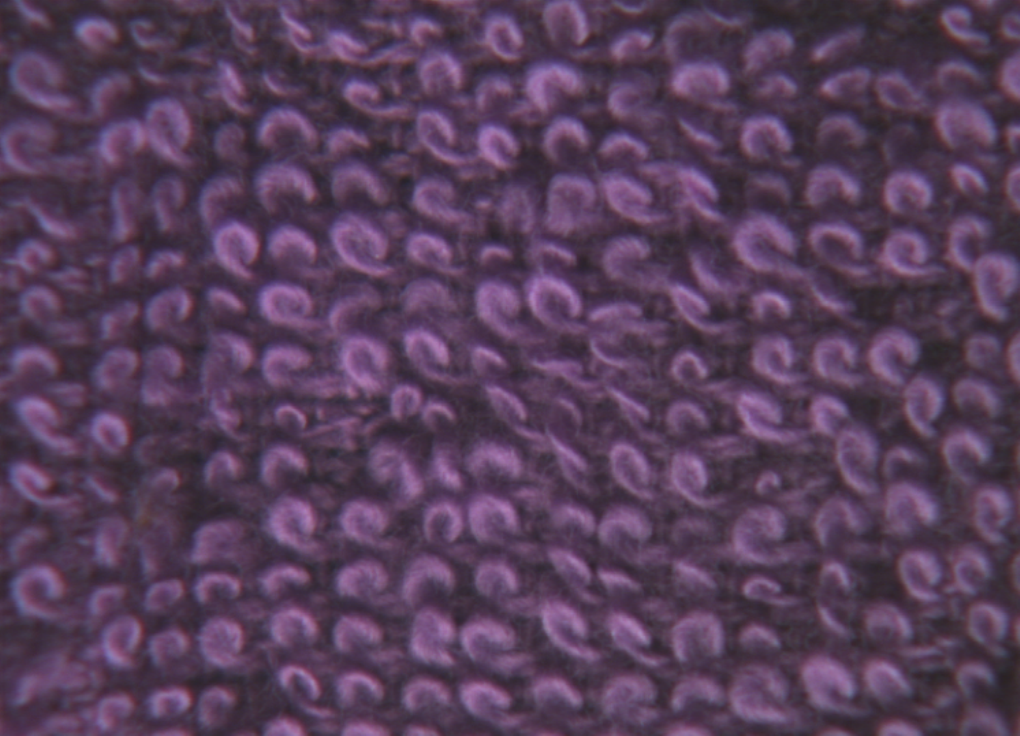} }}%
    \qquad \\
    \subfloat{{\includegraphics[width=4cm]{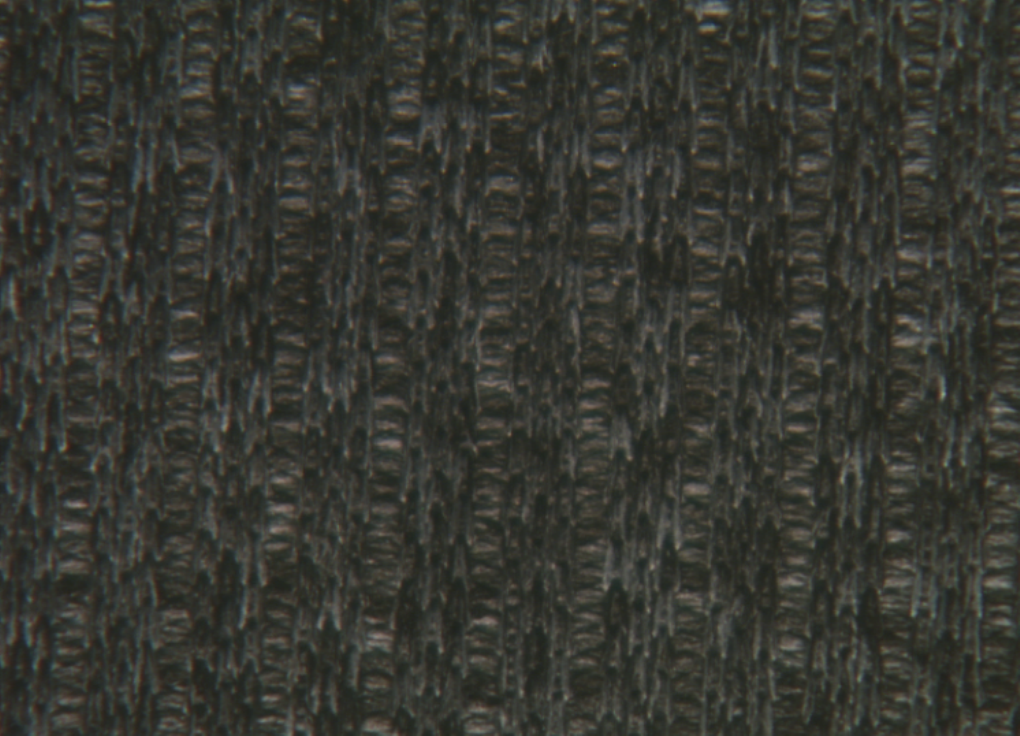} }}%
    \qquad
    \subfloat{{\includegraphics[width=4cm]{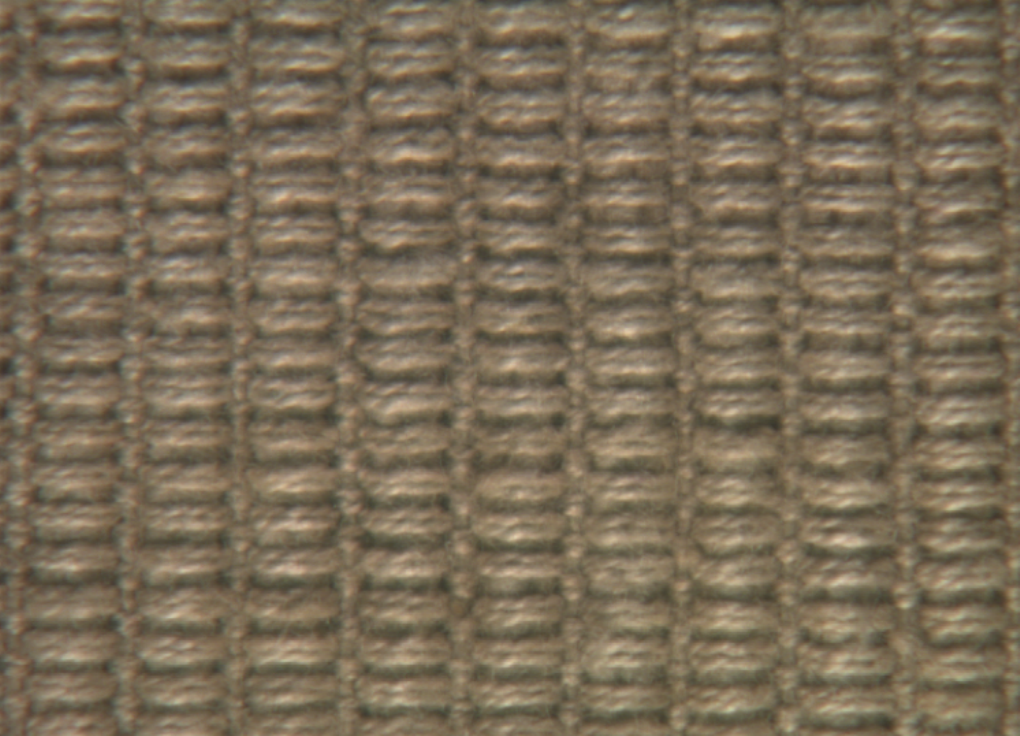} }}%

    \caption{{\bf Textures.} The {\it Outex\_TC\_00000} dataset of surface textures~\cite{ojala2002outex}.
The data set consists of 240 images for training, 240 for testing, as prescribed by the test suite.
Each sample carries one out of 24 labels.
Above shows texture of four different labels.}
    \label{fig:outex}%
\end{figure}

\begin{figure}[h!]%
\centering
\includegraphics[width=7cm]{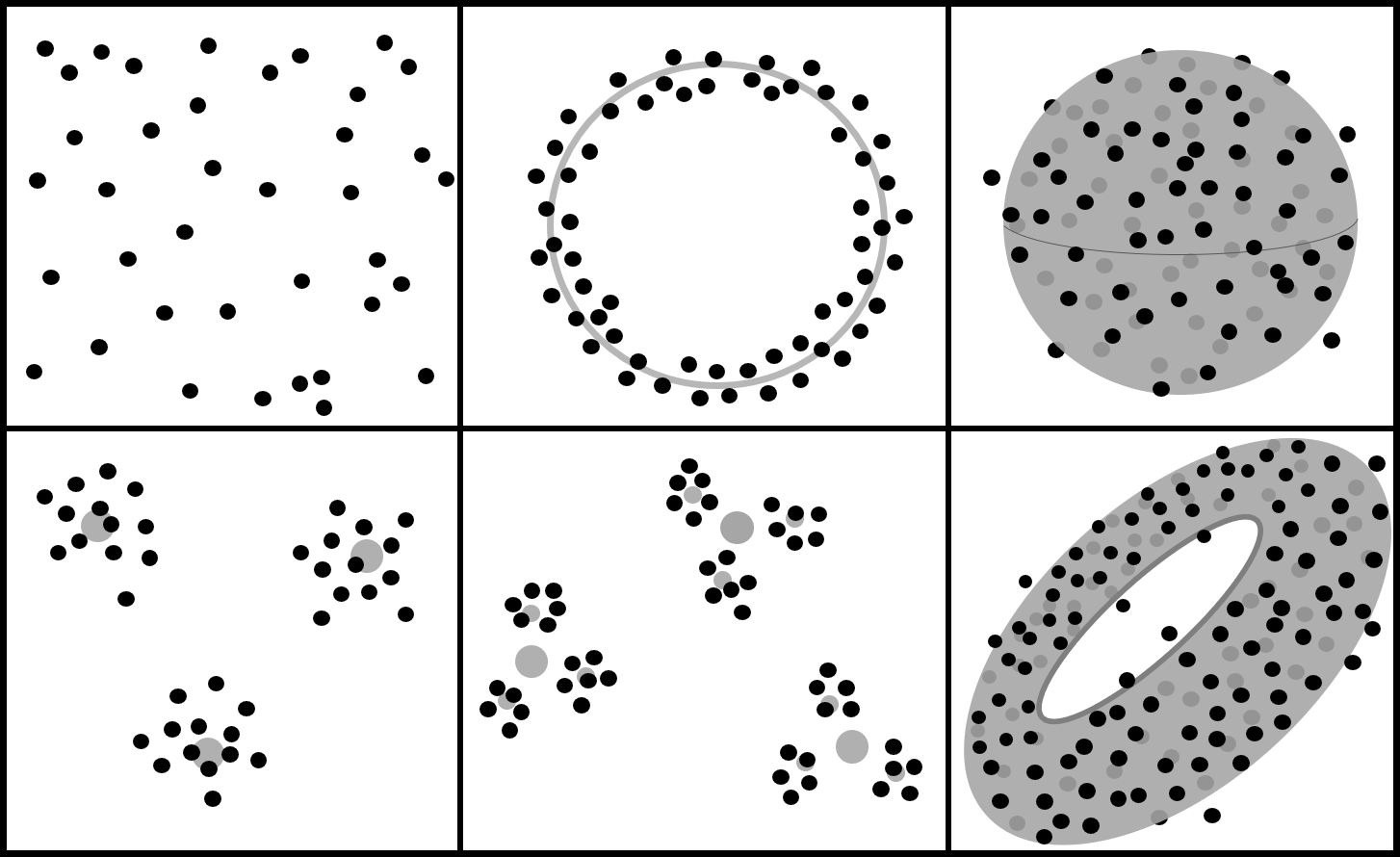} 
\caption{{\bf Shapes.}
The synthetic shapes dataset~\cite{adams2017persistence, JZZ18} consists of point cloud data with six labels: random, circle, sphere, clusters, clusters-of-clusters and torus.
Within each label, there are 50 point clouds containing 500 points each.
The Gaussian noise model used in each case has standard deviation 0.1.
We use a 50\% to 50\% test-train split and the persistence diagrams provided by~\cite{JZZ18} \url{https://github.com/bziiuj/pcodebooks}.
Note that our learning method is different to that of~\cite{adams2017persistence, JZZ18} where $K$-medoids clustering is used.
}
\label{fig:shapes}%
\end{figure}

\section{Conclusion}
We have presented a two-step methodology to build a feature map for barcodes: the first step identifies a barcode as a persistence path; the second step computes the signature of this path.
The motivation for the first step is that it captures the dynamic, time-indexed nature of topological persistence.
In other words, just as important as the appearance and disappearance of topological features across various scales is the {\em order} in which these events occur.
The motivation for the second step is that the signature provides an extremely well-studied feature map for paths.
The four central advantages of this approach are:
\begin{description}

\item [Theory]
Barcodes are represented as elements in the tensor algebra given as iterated integrals.
The tensor algebra has a rich and well-understood algebraic structure that makes it easy to establish both universality and characteristicness.
Moreover, this allows tools from stochastic analysis and non-commutative algebra to interpret barcodes produced by topological data analysis.

\item [Performance]
We benchmarked our feature maps on standard datasets against two recent, very competitive methods (one kernelized, one unkernelized) and achieved state-of-the art performance on two out of the three benchmarks. 

\item [Flexibility]
Persistence path embeddings differ in stability, computability, and discriminative power.
This flexibility to emphasize one aspect over the other can prove decisive for a given dataset. 
For example, the Betti and Euler embeddings provide a large dimensionality reduction, with the Euler embedding yielding further computational advantages since the whole barcode never has to be computed.
In this case, the price is a reduced discriminative power, which was not relevant to the benchmark datasets considered here.

\item [(Un)kernelized learning]
We provide both kernelized and unkernelized versions of our feature maps.
The former allows us to use the well-developed modular kernel learning approach that provides statistical learning guarantees.
The latter allows to use far more general classifiers (e.g., those which use regularization beyond Tikhonov to avoid overfitting), and hence potentially deal with large scale problems that are beyond the current reach of kernel-based methods.

\end{description}

\bibliographystyle{abbrv}
{\small
\bibliography{rp}}

\def\cprime{$'$}
\begin{thebibliography}{10}

\bibitem{2015arXiv150706217A}
H.~{Adams}, S.~{Chepushtanova}, T.~{Emerson}, E.~{Hanson}, M.~{Kirby},
  F.~{Motta}, R.~{Neville}, C.~{Peterson}, P.~{Shipman}, and L.~{Ziegelmeier}.
\newblock {Persistence Images: A Stable Vector Representation of Persistent
  Homology}.
\newblock {\em The Journal of Machine Learning Research}, 18(1):218--252, 2017.

\bibitem{adams2017persistence}
H.~Adams, T.~Emerson, M.~Kirby, R.~Neville, C.~Peterson, P.~Shipman,
  S.~Chepushtanova, E.~Hanson, F.~Motta, and L.~Ziegelmeier.
\newblock Persistence images: A stable vector representation of persistent
  homology.
\newblock {\em The Journal of Machine Learning Research}, 18(1):218--252, 2017.

\bibitem{adcock2016ring}
A.~Adcock, E.~Carlsson, and G.~Carlsson.
\newblock The ring of algebraic functions on persistence bar codes.
\newblock {\em Homology Homotopy Appl.}, 18(1):381--402, 2016.

\bibitem{BGLY16}
H.~Boedihardjo, X.~Geng, T.~Lyons, and D.~Yang.
\newblock The signature of a rough path: uniqueness.
\newblock {\em Adv. Math.}, 293:720--737, 2016.

\bibitem{Bubenik15}
P.~{Bubenik}.
\newblock {Statistical topological data analysis using persistence landscapes}.
\newblock {\em Journal of Machine Learning Research}, 16:77–102, 2015.

\bibitem{landtool}
P.~Bubenik and P.~D\l~otko.
\newblock A persistence landscapes toolbox for topological statistics.
\newblock {\em J. Symbolic Comput.}, 78:91--114, 2017.

\bibitem{2017arXiv170603358C}
M.~Carri{\`e}re, M.~Cuturi, and S.~Oudot.
\newblock Sliced {W}asserstein kernel for persistence diagrams.
\newblock In D.~Precup and Y.~W. Teh, editors, {\em Proceedings of the 34th
  International Conference on Machine Learning}, volume~70 of {\em Proceedings
  of Machine Learning Research}, pages 664--673, International Convention
  Centre, Sydney, Australia, 06--11 Aug 2017. PMLR.

\bibitem{CDLL16}
T.~Cass, B.~K. Driver, N.~Lim, and C.~Litterer.
\newblock On the integration of weakly geometric rough paths.
\newblock {\em J. Math. Soc. Japan}, 68(4):1505--1524, 2016.

\bibitem{cdsgo}
F.~Chazal, V.~de~{S}ilva, M.~Glisse, and S.~Oudot.
\newblock {\em Structure and stability of persistence modules}.
\newblock SpringerBriefs in Mathematics. Springer, 2016.

\bibitem{chevyrev2016characteristic}
I.~Chevyrev and T.~Lyons.
\newblock Characteristic functions of measures on geometric rough paths.
\newblock {\em Ann. Probab.}, 44(6):4049--4082, 2016.

\bibitem{CO18Sigs}
I.~{Chevyrev} and H.~{Oberhauser}.
\newblock {Signature moments to characterize laws of stochastic processes}.
\newblock {\em ArXiv e-prints}, Oct. 2018.

\bibitem{cseh}
D.~Cohen-Steiner, H.~Edelsbrunner, and J.~Harer.
\newblock Stability of persistence diagrams.
\newblock {\em Discrete and Computational Geometry}, 37(1):103--120, 2007.

\bibitem{dsg}
V.~de~Silva and R.~Ghrist.
\newblock Coverage in sensor networks via persistent homology.
\newblock {\em Algebraic and Geometric Topology}, 7(1):339--358, 2007.

\bibitem{di2015comparing}
B.~Di~Fabio and M.~Ferri.
\newblock Comparing persistence diagrams through complex vectors.
\newblock In {\em Image analysis and processing---{ICIAP} 2015. {P}art {I}},
  volume 9279 of {\em Lecture Notes in Comput. Sci.}, pages 294--305. Springer,
  Cham, 2015.

\bibitem{friz-victoir-book}
P.~K. Friz and N.~B. Victoir.
\newblock {\em Multidimensional stochastic processes as rough paths: theory and
  applications}.
\newblock Cambridge Studies in Advanced Mathematics. Cambridge University
  Press, Cambridge, 2010.

\bibitem{Gameiro15}
M.~Gameiro, Y.~Hiraoka, S.~Izumi, M.~Kramar, K.~Mischaikow, and V.~Nanda.
\newblock A topological measurement of protein compressibility.
\newblock {\em Japan Journal of Industrial and Applied Mathematics},
  32(1):1--17, 2015.

\bibitem{ghrist:08}
R.~Ghrist.
\newblock Barcodes: the persistent topology of data.
\newblock {\em Bulletin of the American Mathematical Society}, 45(1):61--75,
  2008.

\bibitem{hg}
R.~Ghrist and G.~Henselman.
\newblock Matroid filtrations and computational persistent homology.
\newblock {\em arXiv:1606.00199v2[math.at]}, 2016.

\bibitem{gromov}
M.~Gromov.
\newblock Groups of polynomial growth and expanding maps.
\newblock {\em Inst. Hautes Études Sci. Publ. Math.}, 53:53–73, 1981.

\bibitem{hatcher}
A.~Hatcher.
\newblock {\em Algebraic topology}.
\newblock Cambridge University Press, 2002.

\bibitem{kahle:rs}
M.~Kahle.
\newblock Topology of random simplicial complexes: a survey.
\newblock {\em AMS Contemporary Math}, 620:201--222, 2014.

\bibitem{2016arXiv160108169K}
F.~J. {Kir{\'a}ly} and H.~{Oberhauser}.
\newblock {Kernels for sequentially ordered data}.
\newblock {\em ArXiv e-prints, 1601.08169}, 2016.

\bibitem{gait}
J.~Lamar-Leon, R.~Alonso-Baryolo, E.~Garcia-Reyes, and R.~Gonzalez-Diaz.
\newblock Persistent homology-based gait recognition robust to upper body
  variations.
\newblock In {\em 2016 23rd International Conference on Pattern Recognition
  (ICPR)}, pages 1083--1088, Dec 2016.

\bibitem{lyons-04}
T.~J. Lyons, M.~Caruana, and T.~L{\'e}vy.
\newblock Differential equations driven by rough paths, 2007.
\newblock Lectures from the 34th Summer School on Probability Theory held in
  Saint-Flour, July 6--24, 2004, With an introduction concerning the Summer
  School by Jean Picard.

\bibitem{mors}
K.~Mischaikow and V.~Nanda.
\newblock Morse theory for filtrations and efficient computation of persistent
  homology.
\newblock {\em Discrete and Computational Geometry}, 50(2):330--353, 2013.

\bibitem{ns}
V.~Nanda and R.~Sazdanovi\'c.
\newblock Simplicial models and topological inference for biological systems.
\newblock In N.~Jonoska and M.~Saito, editors, {\em Discrete and Topological
  Models in Molecular Biology}, chapter~6, pages 109--141. Springer, 2014.

\bibitem{ojala2002outex}
T.~Ojala, T.~Maenpaa, M.~Pietikainen, J.~Viertola, J.~Kyllonen, and
  S.~Huovinen.
\newblock Outex-new framework for empirical evaluation of texture analysis
  algorithms.
\newblock In {\em Pattern Recognition, 2002. Proceedings. 16th International
  Conference on}, volume~1, pages 701--706. IEEE, 2002.

\bibitem{oudot}
S.~Oudot.
\newblock {\em Persistence theory: from quiver representations to data
  analysis}, volume 209 of {\em Mathematical Surveys and Monographs}.
\newblock American Mathematical Society, 2015.

\bibitem{scikit-learn}
F.~Pedregosa, G.~Varoquaux, A.~Gramfort, V.~Michel, B.~Thirion, O.~Grisel,
  M.~Blondel, P.~Prettenhofer, R.~Weiss, V.~Dubourg, J.~Vanderplas, A.~Passos,
  D.~Cournapeau, M.~Brucher, M.~Perrot, and E.~Duchesnay.
\newblock Scikit-learn: Machine learning in {P}ython.
\newblock {\em Journal of Machine Learning Research}, 12:2825--2830, 2011.

\bibitem{Perea2015}
J.~A. Perea and J.~Harer.
\newblock Sliding windows and persistence: An application of topological
  methods to signal analysis.
\newblock {\em Foundations of Computational Mathematics}, 15(3):799--838, Jun
  2015.

\bibitem{SGS18}
C.-J. Simon-Gabriel and B.~Sch{\"o}lkopf.
\newblock Kernel distribution embeddings: Universal kernels, characteristic
  kernels and kernel metrics on distributions.
\newblock {\em Journal of Machine Learning Research}, 19(44):1--29, 2018.

\bibitem{CosmicWeb}
T.~Sousbie.
\newblock The persistent cosmic web and its filamentary structure – {I}.
  {T}heory and implementation.
\newblock {\em Monthly Notices of the Royal Astronomical Society},
  414(1):350--383, 2011.

\bibitem{pht}
K.~Turner, S.~Mukherjee, and D.~M. Boyer.
\newblock Persistent homology transform for modeling shapes and surfaces.
\newblock {\em Information and Inference: A Journal of the IMA}, 3(4):310--344,
  2014.

\bibitem{fullerene}
K.~Xia, X.~Feng, Y.~Tong, and G.~W. Wei.
\newblock Persistent homology for the quantitative prediction of fullerene
  stability.
\newblock {\em Journal of Computational Chemistry}, 36:408--422, 2015.

\bibitem{JZZ18}
B.~{Zielinski}, M.~{Juda}, and M.~{Zeppelzauer}.
\newblock {Persistence Codebooks for Topological Data Analysis}.
\newblock {\em ArXiv e-prints}, Feb. 2018.

\bibitem{zc}
A.~Zomorodian and G.~Carlsson.
\newblock Computing persistent homology.
\newblock {\em Discrete and Computational Geometry}, 33:249--274, 2005.

\end{thebibliography}
\end{document}